\documentclass[12pt]{article} 
\usepackage{parskip}
\usepackage{microtype}
\usepackage{graphicx}
\usepackage{booktabs} 
\usepackage{tablefootnote}
\usepackage[margin=1in]{geometry}

\usepackage{float} 

\usepackage{algorithm}
\usepackage{amsmath}
\usepackage[numbers]{natbib}

\usepackage[noend]{algpseudocode} 
\restylefloat{algorithm}
\usepackage[most]{tcolorbox}
\usepackage{hyperref} 
\usepackage{url}
\hypersetup{
    colorlinks,
    linkcolor={blue},
    citecolor={blue},
    urlcolor={blue}
}
\usepackage{amsthm}
\usepackage{epstopdf}
\graphicspath{{Figures/}}
\usepackage{wrapfig}
\algrenewcommand\algorithmicrequire{\textbf{Input:}}

\usepackage{color-edits} 
\usepackage{nicefrac} 
\usepackage{authblk}

\usepackage{amssymb}  
\usepackage{verbatim} 
\usepackage{amsxtra}  

\usepackage{enumerate}

\usepackage{amsfonts}
\usepackage{color}
\usepackage{bbm}

\usepackage{subcaption}

\usepackage{wasysym}

\usepackage{letltxmacro}

\def\DHLhksqrt#1#2{%
\setbox0=\hbox{$#1\sqrt{#2\,}$}\dimen0=\ht0
\advance\dimen0-0.2\ht0
\setbox2=\hbox{\vrule height\ht0 depth -\dimen0}%
{\box0\lower0.4pt\box2}}

\usepackage{mathtools}
\mathtoolsset{showonlyrefs=true}

\newcommand{\isi}[1]{}
\newcommand{\isiincl}[2]{}

\newcommand{\googleincl}[2]{}
\newcommand{\googleinclabs}[3]{}

\usepackage{tikz}
\usetikzlibrary{shapes,arrows,decorations.markings,positioning}

\tikzstyle{plant} = [draw, fill=red!5, rectangle, 
    minimum height=3em, minimum width=6em]
\tikzstyle{block} = [draw, fill=blue!5, rectangle, 
    minimum height=3em, minimum width=6em]
\tikzstyle{sum} = [draw, fill=yellow!10, circle, node distance=1cm]
\tikzstyle{coord} = [coordinate]
\tikzstyle{gain} = [draw, fill=red!5, regular polygon, regular polygon sides=3, shape border rotate=-90]
\tikzstyle{pinstyle} = [pin edge={to-,thick,black}]

\tikzstyle{BitPipe} = [thick, decoration={markings,mark=at position
   1 with {\arrow[semithick]{open triangle 60}}},
   double distance=1.4pt, shorten >= 5.5pt,
   preaction = {decorate},
   postaction = {draw,line width=1.4pt, white,shorten >= 4.5pt}]
   
\usetikzlibrary{fit,backgrounds}

\usepackage{epsfig, graphicx, psfrag}

\usepackage[mathcal]{euscript}
\usepackage[cal=boondox]{mathalfa}

\DeclareMathAlphabet{\pazocal}{OMS}{zplm}{m}{n}

\usepackage{amsthm}

\newtheorem{thm}{Theorem}

\newtheorem{lem}{Lemma}

\newtheorem{corol}{Corollary}
\newtheorem{prop}{Proposition}

\theoremstyle{definition}
\newtheorem{defn}{Definition}
\newtheorem*{defn*}{Definition}

\newtheorem*{scheme*}{Scheme}

\theoremstyle{remark}
\newtheorem{remark}{Remark}
\newtheorem{assump}{Assumption}

\providecommand{\thmref}[1]{Th.~\ref{#1}}

\providecommand{\secref}[1]{Sec.~\ref{#1}}
\providecommand{\lemref}[1]{Lem.~\ref{#1}}
\providecommand{\propref}[1]{Prop.~\ref{#1}}

\providecommand{\figref}[1]{Fig.~\ref{#1}}

\providecommand{\corolref}[1]{Corol.~\ref{#1}}
\providecommand{\appref}[1]{App.~\ref{#1}}
\providecommand{\assref}[1]{Assump.~\ref{#1}}

\providecommand{\tblref}[1]{Tbl.~\ref{#1}}

\newcommand{\ie}{i.e.}

\newcommand{\reals}{\mathbb{R}}

\newcommand{\mD}{\mathcal{D}}

\newcommand{\old}[1]{}
\newcommand{\rem}[1]{}

\newcommand{\eps}{{\epsilon}}

\newcommand{\ty}{\tilde y}

\providecommand{\brI}{\boldsymbol{\mathrm{I}}}

\providecommand{\brF}{\boldsymbol{\mathrm{F}}}
\providecommand{\brH}{\boldsymbol{\mathrm{H}}}
\providecommand{\brR}{\boldsymbol{\mathrm{R}}}
\providecommand{\brA}{\boldsymbol{\mathrm{A}}}
\providecommand{\brC}{\boldsymbol{\mathrm{C}}}
\providecommand{\brD}{\boldsymbol{\mathrm{D}}}

\newcommand{\sx}{{\mathsf{x}}}

\newcommand{\sy}{{\mathsf{y}}}
\newcommand{\sz}{{\mathsf{z}}}

\newcommand{\abs}[1]{\left| #1 \right|}
\newcommand{\Norm}[1]{\left\| #1 \right\|}

\providecommand{\comment}[1]{}

\providecommand{\norm}[1]{\Norm{#1}}

\newcommand{\beqn}[1]{\begin{eqnarray}\label{#1}}
\newcommand{\eeqn}{\end{eqnarray}}
\newcommand{\beq}[1]{\begin{equation}\label{#1}}
\newcommand{\eeq}{\end{equation}}

\newcommand{\argmin}{\mathrm{argmin}}

\makeatletter
\newcommand{\vast}{\bBigg@{4}}
\newcommand{\Vast}{\bBigg@{5}}
\makeatother

\newcommand{\defeq}{\triangleq}

\providecommand{\bbE}{\mathbb{E}}

\providecommand{\E}[1]{\bbE \left[ #1 \right]}

\providecommand{\Esub}[2]{\bbE_{#1} \left[ #2 \right]}

\providecommand{\CP}[2]{P \left( #1 \middle| #2 \right)}

\newcommand{\pW}{\pazocal{W}}

\newcommand{\pY}{\pazocal{Y}}

\renewcommand{\epsilon}{\varepsilon}
\renewcommand{\eps}{\varepsilon}

\title{The Approximate Fisher Influence Function: \\ Faster Estimation of Data Influence in Statistical Models } 

\author{Omri Lev and Ashia C. Wilson \\
MIT
}

\date{} 

\begin{document}

\maketitle 

\begin{abstract}
    Quantifying the influence of infinitesimal changes in training data on model performance is crucial for understanding and improving machine learning models. In this work, we reformulate this problem as a weighted empirical risk minimization and enhance existing influence function-based methods by using information geometry to derive a new algorithm to estimate influence. Our formulation proves versatile across various applications, and we further demonstrate in simulations how it remains informative even in non-convex cases. Furthermore, we show that our method offers significant computational advantages over current Newton step-based methods.
\end{abstract}

\begingroup
\renewcommand\thefootnote{} %
\footnotetext{Corresponding Author: $\texttt{omrilev@mit.edu}$} 
\endgroup

\section{Introduction}
Understanding how a model's behavior changes with slight modifications to its training data is crucial for numerous machine-learning applications. These include detecting harmful patterns and constructing adversarial examples \citep{KohLiang_Influence_DL, koh2019accuracy, basu2020influence}, conducting efficient cross-validation (CV) for model assessment and model selection \citep{Beirami_ACV, Wilson_OptimizerComparison}, enabling data unlearning without full retraining \citep{Sekhari_Unlearning_Neurips, Wilson_OptimizerComparison}, and evaluating robustness to data-dropping \citep{Broderick_Robustness}, among others. A common foundation for these tasks is the use of second-order approximations to capture the model’s sensitivity to training data perturbations.

The most widely used technique in this space involves the Newton step, leveraging a gradient preconditioned by the inverse Hessian matrix. However, this approach can be computationally prohibitive and numerically unstable, particularly in high-dimensional and non-convex scenarios \citep{Limitations_FisherNGD, Bae_PBRF_Influence}. As a computationally lighter alternative, several studies have explored influence approximations based on variants of the Fisher Information Matrix (FIM) \citep{singh2020woodfisher, sattigeri2022fair, Bae_PBRF_Influence, grosse2023studying, choe2024_data_GPT}. Yet, despite growing empirical adoption, there remains a lack of theoretical understanding to guide the selection and use of FIM variants across diverse applications. Many of these works rely exclusively on the empirical FIM, which is known to underperform in several settings. Moreover, prior theoretical analyses of influence functions have largely assumed smooth, differentiable regularization—most commonly classical $L_2$—which limits their applicability in practical settings.
Indeed, modern machine learning models frequently incorporate non-differentiable regularizers (e.g., $\ell_1$ or group sparsity penalties), and recent work has shown that even certain neural networks can be framed as convex optimization problems with non-smooth regularization terms \citep{Pilanci_ICML2020, Zeger2024_HiddenConvexity}. This motivates the need for influence methods that are not only theoretically grounded but also scalable and valid under general, possibly non-smooth, regularizers—a gap that our work addresses.

In this paper, we propose the Approximate Fisher Influence Function (AFIF), a practical and theoretically justified framework for estimating influence in statistical models. AFIF leverages an approximation of the Fisher Information Matrix derived from exponential family structure, offering a computationally efficient alternative to Hessian-based methods. In contrast to prior influence techniques, which struggle with general regularization and lack formal guarantees for FIM-based approximations, our approach is provably accurate in convex settings and supports a broad range of regularization types—including non-differentiable ones.

Our main contributions are summarized as follows:
\begin{itemize}
    \item \textbf{General Influence Estimation with Theoretical Guarantees:}
We develop the Approximate Fisher Influence Function (AFIF), a theoretically grounded method for influence estimation that supports general regularization, including non-differentiable terms. Our analysis establishes the first theoretical guarantees for using FIM-based approximations in key influence tasks such as cross-validation and fairness evaluation—extending prior work limited to smooth, $L_2$-regularized settings.  
    \item \textbf{Scalable FIM Approximation with Strong Empirical Performance:} We introduce a novel FIM variant derived from the exponential family structure that is both computationally efficient and theoretically justified. This formulation provides practical guidance on FIM selection. Empirically, AFIF matches the accuracy of Hessian-based methods while offering substantial improvements in speed and stability across diverse models and tasks. 
\end{itemize}
\noindent \textbf{Notation:} 
Random variables are represented by sans-serif fonts ($\sx, \sy, \sz$), and their realizations by regular italics ($x, y, z$). The PDF of $\sz$ is $P_{\sz}(\cdot)$. 
Sets of values are indicated by capital calligraphic letters, such as $\mD \triangleq \left\{z_1, z_2, \ldots, z_n\right\}$.  Matrices are in bold capitals, with $\brI_d$ as the $d \times d$ identity matrix. 
We use $f(x) = o(g(x))$ and  $f(x) = O(g(x))$ when $ f(x)/g(x) = 0$ and $f(x)/g(x) = c \neq 0$ in the limit $x \rightarrow \infty$. 
We denote the Lipschitz constant of a function $f$ by $\text{Lip}(f) \triangleq{\sup} \{ \norm{f(x) - f(y)}/\norm{x - y}: x\ne y \in \text{supp}(f)\}$. The inner product between two vectors $\theta_1$ and $\theta_2$ is denoted by $\theta^{\top}_1 \theta_2 \triangleq \langle \theta_1, \theta_2 \rangle $.
\section{Problem Statement}
\label{s:ProblemDef}
Given a dataset $\mD = (z_1, z_2, \dots z_n)$ where each $z_i$ is comprised of a covariate $x_i$ and a label or response $y_i$, it is commonplace to use empirical risk minimization (ERM) to obtain a predictive model to deploy. 
In this work, we consider the problem of \emph{weighted} ERM (wERM), \ie \ given 
a loss function $\ell(\cdot)$, a regularizer $\pi(\cdot)$, a regularization parameter $\lambda \in [0, \infty)$ and a set of weights $w^n \triangleq (w_1,\ldots, w_n)$, our goal is to solve for $\hat{\theta}(w^{n})$ that is defined as 
\begin{align}
    \label{eq:WERM_Def}
    \hat{\theta}(w^{n}) &\triangleq \underset{\theta}{\argmin} \ L(\mD, \theta, \lambda, w^n),\\
    L(\mD, \theta, \lambda, w^n) &\triangleq \frac{1}{n}\sum_{i=1}^{n}w_i\ell(z_i, \theta) + \lambda \pi(\theta).
\end{align} 
This formulation is equivalent to classical ERM when $w^{n} = (1, \ldots, 1) \triangleq \boldsymbol{1}$, whose solution is denoted by $\hat{\theta}(\boldsymbol{1})$\ \footnote{Throughout, we simplify our notation by omitting the explicit dependence on $\lambda$ when possible. For example, we write $ L(\mD, \theta, w^n)$ instead of $L(\mD, \theta, \lambda, w^n)$ whenever $\lambda = 0$.}. 

In many scenarios $\ell(z,\theta) = -\log(\CP{y}{f(x;\theta)})$; that is, the loss can be interpreted as a negative log-likelihood under a probabilistic model induced by a parameterized function $f(x;\theta)$, often taken to be a neural network. Moreover, we study the case where $\CP{y}{f(x;\theta)}$ belongs to an \emph{exponential-family} \citep{Wainright_ExponentialFamily} whose natural parameters are $f(x;\theta)$ and whose natural statistics are denoted by $t(y)$, namely, $\log(\CP{y}{f(x;\theta)}) = f(x;\theta)^{\top}t(y) - \log(\sum_{\tilde{y}\in \pazocal{Y}}\exp(f(x;\theta)^{\top}t(\tilde{y}))) + \beta(y)$ for some function $\beta(y)$. This is satisfied by many common loss functions in machine learning (see popular examples for such losses in \appref{app:exp_family_loss}).
\begin{remark}
    Following \cite{banerjee2005clustering}, this class of losses corresponds to loss functions that can be captured by a Bregman divergence up to an additional term that is independent of $f(x;\theta)$. See further discussion in \appref{app:exp_family_loss}.  
\end{remark}
\subsection{Inference Objective} We study the \emph{inference objective}, $T(\cdot, \cdot): \reals^{d} \times \reals^{n} \to \reals^{k}$, which maps a parameter vector $\theta$ and a weight vector $w^n$ to a desired inference target,   
where $w^n$ belongs to a family of weight vectors $\pW$. In particular, we focus on cases where $w^n$ corresponds to a leave-one-out weight vector, defined as  
\[
\mathcal{D}^{-i} \triangleq \{w^n : w_j = \mathbbm{1}\{j \ne i\} \}.
\]  
This formulation captures a range of tasks, including:

\paragraph{Cross Validation} To assess and select models, leave-one-out cross-validation (LOOCV) estimates model performance by iterative training on all but one data point and evaluating on the omitted instance. Specifically, for each $i$, it computes the evaluation metric:
\[
    T(\hat{\theta}(w^n), w^n) \triangleq \frac{1}{n}\ell(z_i, \hat{\theta}(w^n)) \quad \text{for} \quad w^n \in \mathcal{D}^{-i},
\]
where $\mathcal{D}^{-i}$ denotes the leave-one-out weight vectors, $\hat{\theta}(w^n)$ is the model trained with $w^n$, and the corresponding evaluation is taken on the omitted sample. Leave-$k$-out cross-validation follows analogously by removing a subset of $k$ observations in each iteration \citep{geisser1975predictive, Stone_CV}.


\paragraph{Machine Unlearning} 
To remove the influence of a data point $z_i$, the ``unlearned model'' is obtained by computing  
\[
    T(\hat{\theta}(w^n), w^n) = \hat{\theta}(w^n) \quad \text{for} \quad w^n \in \mathcal{D}^{-i}.
\]  
This ensures that the model parameters are updated as if $z_i$ were never included in training \citep{Unlearning_CaoYang, UnlearningChallenges_2023}. Similarly, unlearning $k$ data points follows the same formulation using leave-$k$-out weight vectors.
     
 \paragraph{Data attribution} 
Understanding the contribution of a training sample $z_i \in \mathcal{D}$ to a model’s prediction on a test point $z_{\text{test}}$ \citep{KohLiang_Influence_DL} is formulated as  comparing $\ell(z_{\text{test}}, \hat{\theta}(\boldsymbol{1}))$ with
\begin{align}
    T(\hat{\theta}(w^n), w^n) &= \ell(z_{\text{test}}, \hat{\theta}(w^n))\triangleq T(\hat{\theta}(w^n)) \quad \text{for} \quad w^n \in \mathcal{D}^{-i}.
\end{align}
  Attribution to a set of $k$ points follows analogously.

     
\paragraph{Fairness Evaluation} 
\citet{ghosh2023influence} propose to evaluate the impact of $z_i$ on model fairness by computing $T(\hat{\theta}(w^n))$ for $w^n \in \mathcal{D}^{-i}$, where $T$ is a chosen fairness metric. For example, if we have a dataset $\{x_i\}_{i=1}^{n}$ with binary sensitive attributes $\{s_i\}_{i=1}^{n}$, robustness with respect to \emph{demographic parity}—which assesses whether the model’s predictions are independent of a sensitive attribute $\mathsf{s}$—is given by:
\begin{align}
    \label{eq:DP_Def}
    T(\hat{\theta}(w^n), w^n) &\triangleq T(\hat{\theta}(w^n)) \\
    T(\hat{\theta}(w^n)) &= |\mathbb{E}_{\hat{P}(\sx|\mathsf{s} = 0)}[f(\sx; \hat{\theta}(w^n))] - \mathbb{E}_{\hat{P}(\sx|\mathsf{s} = 1)}[f(\sx; \hat{\theta}(w^n))]|, \quad \text{for } w^n \in \mathcal{D}^{-i}.
\end{align}
Here, $\hat{P}(\sx = x|\mathsf{s} = s)$ is the empirical distribution for $s \in \{0,1\}$ . For cases where the sensitive attributes $\left\{s_i\right\}$ are continuous-valued, an alternative fairness metric can be defined via the $\chi^2$ divergence \cite{Mary_Chi2_Fairness, AbhinShah_Fairness}. A popular choice for such a metric is defined via 
\begin{align}
    \label{eq:chi2_def}
    T(\hat{\theta}(w^n), w^n) &\triangleq T(\hat{\theta}(w^n)) \\
    T(\hat{\theta}(w^n)) &= \chi^2\left(\widehat{P}_{f(\sx;\hat{\theta}(w^n)), \mathsf{s}}\|\widehat{P}_{f(\sx;\hat{\theta}(w^n))}\widehat{P}_{\mathsf{s}}\right), \quad \text{for } w^n \in \mathcal{D}^{-i}.
\end{align}
The impact of removing a subset of $k$ samples is assessed analogously by considering $w^n \in \mathcal{D}^{-K}$.

\subsubsection{Inference Approximation}
Since $\hat{\theta}(w^n)$ for each weight vector is often computationally expensive, many methods approximate the inference objective using quantities derived from $\hat{\theta}(\boldsymbol{1})$. That is, instead of solving for $\hat{\theta}(w^n)$ directly, we use an approximation that combines the known vector $\hat{\theta}(\boldsymbol{1})$ with a function of the weights $w^n$:
\[
\hat{\theta}(w^n) \approx g(\hat{\theta}(\boldsymbol{1}), w^n) \triangleq \tilde{\theta}(w^n).
\]
 Typically, $g(\cdot, \cdot)$ is derived from a Taylor series expansion around $\hat{\theta}(\boldsymbol{1})$, capturing the $p$th-order sensitivity of the model parameters to small perturbations in $w^n$. Depending on $p$, this allows for efficient approximation without requiring full retraining~\cite{giordano2019swiss, giordano2019higher, Wilson_OptimizerComparison}. Two widely used approaches to approximate the inference objective $T(\hat{\theta}(w^n), w^n)$ are:

\begin{enumerate}
    \item \textbf{Plug-in Estimator}:
    This approach directly substitutes the approximation $\tilde{\theta}(w^n)$ into the inference objective:
    \begin{align}
        T(\hat{\theta}(w^n), w^n)  \approx T(\tilde{\theta}(w^{n}), w^n) = T(g(\hat{\theta}(\boldsymbol{1}), w^n), w^n). 
    \end{align}
    
    \item  \textbf{Linearized Influence Approximation}:
   Instead of replacing $\hat{\theta}(w^n)$ directly, this method uses a first-order expansion of $T(\hat{\theta}(w^n), w^n)$ around $\hat{\theta}(\boldsymbol{1})$. The approximation function $g(\cdot, \cdot)$ is then incorporated into this expansion to estimate $\hat{\theta}(w^n)$:
   \begin{align}
        \label{eq:TaylorSeries_Objective}
        T(\hat{\theta}(w^n), w^n) 
        \approx T(\hat{\theta}(\boldsymbol{1}), w^n) + \langle \nabla_{\theta} T(\hat{\theta}(\boldsymbol{1}), w^n),  \tilde{\theta}(w^{n}) - \hat{\theta}(\boldsymbol{1})\rangle. 
    \end{align}
    
\end{enumerate}
Both methods are shown in several works to reduce computational overhead while maintaining strong empirical performance \cite{KohLiang_Influence_DL, koh2019accuracy, Wilson_OptimizerComparison, basu2020influence}. However, the quality of the approximation depends on how well $g(\cdot, \cdot)$ captures the true parameter updates. In the next section, we introduce a new method for creating such an approximation.


\section{The Approximate Fisher Influence Function}
In this section, we introduce our proposed method and describe how it improves the computational efficiency of the currently existing baselines.

A common approach to approximating $\hat{\theta}(w^{n})$ is to optimize a surrogate to the loss function $L(\mD, \theta, \lambda, w^{n})$. This paper focuses on methods based on {\em quadratic approximations of the objective} \citep{cook1980characterizations, KohLiang_Influence_DL, giordano2019swiss, Wilson_OptimizerComparison}, which provide computationally efficient estimates while maintaining accuracy. These approximations yield solutions of the form:
\begin{align}
    \tilde{\theta}(w^{n}) = \hat{\theta}(\boldsymbol{1}) - \brC(\hat{\theta}(\boldsymbol{1}), w^{n}) b(\hat{\theta}(\boldsymbol{1}), w^{n}),
\end{align}
where $b(\cdot, \cdot)$ and $\brC(\cdot, \cdot)$ depend on the specific loss approximation and vary across applications.

A notable instance of this framework is the \emph{infinitesimal jackknife} (IJ) approximation \cite{giordano2019swiss}, denoted $\tilde{\theta}^{\mathrm{IJ}}(w^n)$, which is defined via a Newton step:
\begin{align}
    \label{eq:G_Definition}
    b(\hat{\theta}(\boldsymbol{1}), w^{n}) &\triangleq \frac{1}{n} \sum_{i=1}^{n} \nabla_{\theta} \ell(z_i, \hat{\theta}(\boldsymbol{1})) (w_i - 1), \\
    \brC(\hat{\theta}(\boldsymbol{1}), \boldsymbol{1}) &= \nabla^{2}_{\theta} L(\mD, \hat{\theta}(\boldsymbol{1}), \boldsymbol {1})^{-1} \triangleq \brH(\hat{\theta}(\boldsymbol{1}), \boldsymbol{1})^{-1}.
\end{align}

In this work, we suggest a modified computationally efficient second-order approximation of $\hat{\theta}(w^{n})$ using the {\em natural gradient}. We consider loss functions $\ell$ that represent the log-likelihood of a parametric probabilistic model, $\ell(z,\theta) = -\log(P_{\sy|\sx}(y|f(x;\theta)))$, where $P_{\sy|\sx}(y|f(x;\theta))$ lies on the probability simplex of the output alphabet $\pY$ and is parameterized by $\theta$ \citep{WhyNGD_Amari, Amari_InfoGeometry_Book, banerjee2005clustering}. As discussed by \cite{banerjee2005clustering} (see also \appref{app:exp_family_loss}), this property holds for a large class of losses in machine learning. While the standard gradient identifies the direction that minimizes the objective based on Euclidean distance, the natural gradient accounts for the underlying geometry (curvature) of the parameter space. This is achieved by pre-multiplying the gradient with the inverse of the FIM, which characterizes the sensitivity of the model's likelihood function to changes in parameters. To that end, the Hessian in \eqref{eq:G_Definition} is replaced by:
\begin{align}
    g_{\sy|\sx} &\triangleq \nabla_{\theta} \log(P_{\sy|\sx; \theta}(\sy|\sx; \hat{\theta}(\boldsymbol{1}))),\\
    \brF(\hat{\theta}(\boldsymbol{1})) &\triangleq \mathbb{E}_{(\sx, \sy) \sim P_{\sx,\sy;\theta = \hat{\theta}(\boldsymbol{1})}}[ g_{\sy|\sx} \cdot g_{\sy|\sx}^\top ],
\end{align}
where $P_{\sy|\sx}$ is the probabilistic model induced by the loss function \citep{NGD_Review}. However, since the covariate distribution $P_{\sx}$ is typically unknown, direct computation of the expectation is infeasible. Instead, we approximate the FIM using empirical estimates, averaging over the observed covariates and leveraging the network structure to evaluate expectations over $P_{\sy|\sx;\theta}$ \citep{NGD_Review, Limitations_FisherNGD}. The resulting \emph{approximate FIM} is given by:
\begin{align}
    \label{eq:Approx_FIM_Calc}
    \brF(\mD, \hat{\theta}(\boldsymbol{1})) \triangleq \frac{1}{n} \sum_{x \in \mD} \mathbb{E}_{\sy \sim P_{\sy|\sx=x;\theta=\hat{\theta}(\boldsymbol{1})}} [g_{\sy|\sx} \cdot g_{\sy|\sx}^\top ].
\end{align}
Using this approximation, we define the \emph{Approximate Fisher Infinitesimal Jackknife} similarly to \eqref{eq:G_Definition}, replacing $\brC$ with the approximate FIM $\brF(\mD, \hat{\theta}(\boldsymbol{1}))$:
\begin{align}
    \label{eq:ApproximateFIM_Introduction}
     \tilde{\theta}^{\text{IJ,AF}}(w^{n}) \triangleq \hat{\theta}(\boldsymbol{1}) - (\brF(\mD,\hat{\theta}(\boldsymbol{1})))^{-1}b(\hat{\theta}(\boldsymbol{1}), w^{n}).
\end{align}

Following classical results \citep{Schraudolph_SecondOrder_VP, NGD_Review}, when the loss function is given by $\ell(z,\theta) = -\log(\CP{y}{f(x;\theta)})$ and $P(y|f)$ belongs to an exponential family, the approximate FIM can be interpreted as a positive semi-definite (PSD) approximation of the Hessian. Specifically, the Hessian satisfies:
\begin{align}
    \brH(\hat{\theta}(\boldsymbol{1}), \boldsymbol{1}) = \brF(\mD, \hat{\theta}(\boldsymbol{1})) + \brR,
\end{align}
where $\brF(\mD, \hat{\theta}(\boldsymbol{1}))$ is guaranteed to be PSD, and the remainder term is given by:
\begin{align}
    \frac{1}{n} \sum^{n}_{i=1} \nabla^{2}_{\theta} f(x_i;\hat{\theta}(\boldsymbol{1})) \nabla_{f} \log(P(y_i | f(x_i;\hat{\theta}(\boldsymbol{1})))).
\end{align}
This remainder term can be non-zero, for example, in cases where $f(x;\theta)$ is non-linear in $\theta$. However, in many settings, including commonly used models, $\brR$ shrinks to zero (in $L_2$ sense) as training accuracy improves \citep{Limitations_FisherNGD, NGD_Review} (see \appref{app:exp_family_loss}, \appref{app:Prof_dist}). Thus, the approximated FIM is often viewed as a computationally efficient PSD approximation of the Hessian.

\subsection{Computational Savings}
\label{s:CompCost}
Here, we show a fundamental efficiency improvement in evaluating \eqref{eq:ApproximateFIM_Introduction} compared to the IJ approach \eqref{eq:G_Definition}. Both methods require computing expressions of the form $\brA^{-1}b(\hat{\theta}(\boldsymbol{1}), w^{n})$ where $\brA = \brF(\mD, \hat{\theta}(\boldsymbol{1}))$ in \eqref{eq:Distance_Approx_eq} and $\brA = \brH(\hat{\theta}(\boldsymbol{1}), \boldsymbol{1})$ for the IJ. Since directly inverting a large $d \times d$ matrix is infeasible, efficient computation of inverse-matrix-vector products is essential. However, since the FIM requires only first-order differentiation through the model, it will typically be much more computationally efficient relative to the Hessian-based alternative. To that end, we will now demonstrate a fundamental computational efficiency of using the FIM over the Hessian in a widely used influence calculation setting invented by the classical work of \citep{KohLiang_Influence_DL} and involves stochastic estimation of inverse-matrix-vector products via the \texttt{LiSSA} algorithm \citep{HazanAgrawal_Lissa}. Similar computational benefits are further applied in modern existing variants of influence measurement techniques that rely on variants of the \texttt{LiSSA} algorithm, and stochastic estimation of inverse matrix-vector products, such as \cite{guo2020fastif, schioppa2022scaling}. 

\subsubsection{Stochastic Estimation}
\label{s:Stochastic_Lissa}
Stochastic estimation techniques rely on generating a sequence of estimators $v_j \triangleq (\widehat{\brA^{-1}x})_j$ that converge in expectation to $\brA^{-1}x$ as $j \to \infty$, where each $v_j$ utilizes only a small batch of training data, yielding a computationally tractable way to estimate $\brA^{-1}x$. As an example of the computational superiority of the FIM-based methods, we will demonstrate the improvement for the celebrated \texttt{LiSSA} algorithm \cite{HazanAgrawal_Lissa} that approximates $\brA^{-1}$ using the truncated Neumann series $\brA^{-1}_j = \sum_{i=0}^{j} (\brI - \sigma\brA)^i$ for some $\sigma > 0$\ \footnote{$\sigma$ is usually a small positive constant to stabilize calculations.}. This approximation is computed via the recursion $\brA^{-1}_j = \brI + (\brI - \sigma\brA)\brA^{-1}_{j-1}$. Consequently, each $v_j$ is defined by $v_j = x + (\brI - \sigma\brA)v_{j-1}$
with $v_0 = x$ and final estimate $v = \sigma v_N$. The major computational hurdle is multiplying by $\brA$. When $\brA$ depends on many training points (e.g., $\brF(\mD, \hat{\theta}(\boldsymbol{1}))$ or $\brH(\hat{\theta}(\boldsymbol{1}), \boldsymbol{1})$), it is typical to estimate it by using a sampled batch of training data. We now analyze the computational complexity of these calculations for each method. 

\paragraph{Estimation with $\brF(\mD, \hat{\theta}(\boldsymbol{1}))$.}
\label{s:Stochastic_FIM}
When $\brA \defeq \brF(\mD, \theta)$, each $\brA v_j$ requires calculating
\begin{align} 
    \label{eq:Esimtates_FIM}
    &\nabla_{\theta} f_i \cdot (\nabla^{2}_{f} \log( \CP{y_i}{f_i}) )(v^{\top}_j\nabla_{\theta} f_i)^{\top})
\end{align}
where $f_i \triangleq f(x_i;\theta)$. Given the form of $\nabla^{2}_{f} \log(\CP{y_j}{f(x_j; \theta)})$ (see \appref{app:exp_family_loss}), computing this expression requires the vector-Jacobian product (VJP) $a_j = v^{\top}_j\nabla_{\theta} f_i$ and the Jacobian-vector product (JVP) $\nabla_{\theta} f_i \cdot (\nabla^{2}_{f} \log( \CP{y_i}{f_i}) )a^{\top}_j)$.

\paragraph{Estimation with $\brH(\hat{\theta}(\boldsymbol{1}), \boldsymbol{1})$.}
\label{s:Stochastic_Hessian}
For $\brA \defeq \brH(\hat{\theta}(\boldsymbol{1}),\boldsymbol{1})$, each $\brA v_j$ requires computing,
\begin{align}
    \label{eq:Stochastic_Hess}
    \nabla^{2}_{\theta}\log(\CP{y_i}{f_i})v_j, 
\end{align}
which requires computing a Hessian-vector product (HVP) with respect to all model parameters.
\paragraph{Comparing Computations.}
Computing \eqref{eq:Stochastic_Hess} requires roughly four evaluations of the entire model \citep{Schraudolph_SecondOrder_VP, iclr2024_hvp}. In contrast, a JVP can be computed in a single forward pass using forward-mode automatic differentiation \citep{JAX_autodiff_cookbook}. Since $\nabla^{2}_{f}\log(\CP{y_i}{f_i})$ is typically simple and depends only on the number of model outputs (not on $d$), evaluating \eqref{eq:Esimtates_FIM} requires just one differentiation in backward mode. Furthermore, given backward differentiation roughly requires twice the complexity of model evaluation \citep{deepspeed_flops, iclr2024_hvp}, this method significantly reduces FLOPs and accelerates computations. We demonstrate these savings through simulations in \secref{s:Exps}. We summarize the results in \tblref{tbl:passes}. 

\setlength{\tabcolsep}{4pt} 
\begin{table}[h!]
\centering
\begin{tabular}{lccc}
\toprule
\textbf{} & \textbf{\shortstack{Forward}} & \textbf{\shortstack{Backward}} & \textbf{FLOPs} \\
\midrule
\eqref{eq:Stochastic_Hess}      & 0   & 2     & $O(4F)$ \\
\eqref{eq:Stochastic_Hess}      & 2   & 1     & $O(4F)$ \\
\eqref{eq:Esimtates_FIM}        & 1   & 1     & $O(3F)$  \\
\bottomrule
\end{tabular}
\caption{Number of differentiations in forward mode, backward mode, and FLOPs required to evaluate \eqref{eq:Esimtates_FIM} and \eqref{eq:Stochastic_Hess}, for different evaluation options from \citep{iclr2024_hvp}. $F$ denotes the FLOPs needed for a single model evaluation.}
\label{tbl:passes}
\end{table}

\begin{remark}
    Although our analysis focuses on the \texttt{LiSSA} algorithm, the fact that the FIM depends solely on first-order gradients means these improvements are broadly applicable to many methods that require differentiating through a large model using the structure of the curvature matrix. For example, similar fundamental gains were observed in \cite{sattigeri2022fair} by employing efficient matrix-inversion techniques based on rank-one updates.
\end{remark}
\section{Theoretical Analysis}
This section presents a general theoretical framework for analyzing the accuracy of inference objective approximations based on plug-in estimates and linearization approximations and based on the FIM. Specifically, we establish conditions under which these approximations remain close, in a well-defined sense, to the true inference function when the loss function satisfies certain regularity properties. While similar results are well understood for infinitesimal jackknife-based approximations, our framework extends these findings to also cover settings when one replaces the Hessian with the approximated FIM.

\subsection{Related work}
Several works have established the accuracy of this approximation under specific conditions on the loss function and the weight vectors $w^{n}$ \citep{giordano2019swiss, Wilson_OptimizerComparison, WilsonSuriyakumar_UnlearningProximal, Sekhari_Unlearning_Neurips}. These results hold under subsets of the following assumptions.

\begin{assump}[Curvature of the Objective]
    \label{ass:one_1}
    For each $i \in [n]$, the function $\frac{1}{n} \ell(z_i, \theta)$ is $\mu$-strongly convex ($\mu > 0$), and the prior $\pi(\theta)$ is convex.
\end{assump}  

\begin{assump}[Lipschitz Hessian of the Objective]
    \label{ass:two_1}
    For each $i \in [n]$, the function $\frac{1}{n} \ell(z_i, \theta)$ is twice differentiable with an $M$-Lipschitz Hessian.
\end{assump}

\begin{assump}[Smooth Hessian of the Objective]
    \label{ass:two_11}
    For each $i \in [n]$, the function $\frac{1}{n} \ell(z_i, \theta)$ is twice differentiable with a $C$-smooth Hessian.
\end{assump}

\begin{assump}[Bounded Moments]
    \label{ass:boundedmoments}
    For given $s, r \geq 0$, the quantity $B_{sr}$ is finite, where
    \begin{align}
        B_{sr} \triangleq \frac{1}{n} \sum_{i=1}^n \text{Lip}(\nabla_\theta \ell(z_i, \cdot))^s \|\nabla_{\theta} \ell(z_i, \hat{\theta}(\boldsymbol{1}))\|^{r}.
    \end{align}
\end{assump}

\begin{assump}[Lipschitz Features]
    \label{ass:LipFeatures}
    The feature mapping $f(x_i; \theta)$ is $C_f$-Lipschitz with a $\tilde{C}_f$-Lipschitz gradient for all $i \in [n]$.
\end{assump}

\begin{assump}[Lipschitz Inference Objective]
    \label{ass:Lipschitz_T}
    The inference objective $T(\theta, w^n)$ is twice differentiable, $C_{T_1}$-Lipschitz, and has a $C_{T_2}$-Lipschitz gradient with respect to $\theta$ for $w^n \in \mathcal{D}^{-i}$ and all $i \in [n]$.
\end{assump}

For the examples in \secref{s:ProblemDef}, the following guarantees were proved to hold under subsets of \assref{ass:one_1}-\assref{ass:Lipschitz_T}: 

\begin{prop}[LOOCV Approximation Bound~(\cite{Wilson_OptimizerComparison}, Thm.~4)]
    \label{prop:CV_Guarantees_Wilson}
    Suppose \assref{ass:one_1}, \assref{ass:two_1}, and \assref{ass:boundedmoments} hold for $(s,r) = \{(0,3), (1,3), (1,4), (1,2), (2,2), (3,2)\}$. When the IJ is used as a plug-in estimate for the LOOCV objective 
    \[
    T_i(\theta) \triangleq T(\theta, w^n) = \frac{1}{n} \ell(z_i, \theta),
    \]
    with $w^n \in \mathcal{D}^{-i}$, the error in this approximation is bounded as
    \begin{align}
        \left| \sum_{i=1}^n \left( T_i(\tilde{\theta}^{\mathrm{IJ}}(w^n)) - T_i(\hat{\theta}(w^n)) \right) \right| = O\left(\frac{M B_{03}}{\mu^3 n^2} + \frac{B_{12}}{\mu^2 n^2} \right).
    \end{align}
\end{prop}

The next proposition relies on the $(\eps,\delta)$-unlearning definition from \citep{Sekhari_Unlearning_Neurips}.
\begin{prop}[Machine Unlearning~\citep{WilsonSuriyakumar_UnlearningProximal}]
    \label{prop:unlearning_wilson}  
    Suppose $\ell(z, \theta)$ is $\mu$-strongly convex, twice differentiable, $L$-Lipschitz, with a $C$-smooth and $M$-Lipschitz Hessian for all $z$, and that $\pi(\theta)$ is convex.\footnote{These assumptions strengthen \assref{ass:one_1}-\assref{ass:boundedmoments}, requiring Lipschitz continuity for any $z$, not just the training samples $\{z_i\}$.} When the IJ is used as a plug-in estimate for the objective $T(\theta) = \theta$, we have
    \begin{align}
        \| T(\tilde{\theta}^{\mathrm{IJ}}(w^n)) - T(\hat{\theta}(w^n))\| \leq \frac{2ML}{n^2\mu^2} + \frac{CL^2}{n^2\mu^3}, \quad \text{for } w^n \in \mathcal{D}^{-i}.
    \end{align}
    Furthermore, the algorithm returning $\tilde{\theta}^{\mathrm{IJ}}(w^n) + \zeta$ for $w^n \in \mathcal{D}^{-i}$ satisfies $(\eps,\delta)$-unlearning, where $\zeta \sim \pazocal{N}(0, c\brI)$ with $
    c = (2\mu ML + CL^2) \sqrt{2\log(5/4\delta)}/\eps \mu^{3} n^{2}.$ 
\end{prop}

\begin{prop}[Data Attribution~(\cite{koh2019accuracy}, Prop.~1)]
    \label{prop:attributions}
    Suppose \assref{ass:one_1}, \assref{ass:two_1}, and \assref{ass:Lipschitz_T} hold, and that $\pi(\theta) = \norm{\theta}^2$. Define $C_{\ell} \triangleq \max_{i\in[n]} \|\nabla \ell(z_i, \hat{\theta}(\boldsymbol{1}))\|$. When the IJ is used as a plug-in estimate for the inference objective 
    \[
    T(\theta) = \ell(z_{\mathrm{test}}, \theta) - \ell(z_{\mathrm{test}}, \hat\theta(\boldsymbol{1})),
    \]
    the approximation error is bounded as
    \begin{align}
        \left|T(\hat{\theta}(w^n)) - T(\tilde{\theta}^{\mathrm{IJ}}(w^n))\right| \leq \frac{MC_{T_1}C^{2}_{\ell}}{n^{2}\mu^{3}}, \quad \text{for } w^n \in \mathcal{D}^{-i}.
    \end{align}
\end{prop}

While certain loss functions may not be Lipschitz, \assref{ass:two_1} and \assref{ass:boundedmoments} require only that the {\em normalized} losses evaluated on the training set satisfy Lipschitz continuity— a condition that generally holds in practice \citep[Assump.~3]{giordano2019swiss}. Similarly, when the inference objective is of the form $\ell(z_{\mathrm{test}}, \theta)$, Lipschitz continuity is required only with respect to the test point $z_{\mathrm{test}}$. As long as $z_{\mathrm{test}}$ is not pathological, this assumption is typically satisfied.\footnote{The assumption that $T$ is Lipschitz is consistent with classical works on influence functions; see \citep[Prop.~1]{koh2019accuracy}.}

Additionally, the framework in \citep{giordano2019swiss} assumes differentiable regularization. In certain cases, similar approximations extend to settings where the regularizer is non-differentiable \citep{Wilson_OptimizerComparison, WilsonSuriyakumar_UnlearningProximal}.

\subsection{The Approximate Fisher Influence Function}
\label{s:MainResult}
We now present the \emph{approximate Fisher influence} and its theoretical characterization. First, we introduce an additional technical assumption about the loss function, which is essential for our proofs.
\begin{assump}
    \label{ass:LossFunc_ExpFamily}
    The loss functions are of the form $\ell(z, \theta) = -\log(\CP{y}{f(x;\theta)})$ where $\CP{y}{f}$ belongs to a regular exponential family whose natural parameters are $f(x;\theta)$. Moreover, we further assume that $\norm{\nabla^{2}_{f}\log\left(\CP{y}{f(x;\theta)}\right)} \leq Q$ \ for some $Q > 0$.
\end{assump}

To accommodate non-smooth regularizers, we utilize the \emph{proximal operator}, defined as:
\begin{align}
    \text{prox}^{\brD}_{\lambda \pi}(v) \triangleq \underset{\theta}{\argmin} \ \left\{(v - \theta)^{\top}\brD(v - \theta) + 2\lambda \pi(\theta)\right\}.\notag
\end{align}
Our main lemma, \lemref{lem:Distance}, defines the approximate Fisher influence and bounds its discrepancy from $\hat{\theta}(w^n)$ for $w^n \in \mathcal{D}^{-i}$.

\begin{lem}
\label{lem:Distance}
Suppose \assref{ass:one_1}, \assref{ass:two_1}, \assref{ass:LipFeatures}, and \assref{ass:LossFunc_ExpFamily} hold. Define
$\bar{E}_n \triangleq \sum^{n}_{j=1}\|\nabla_{f}\log(P(y_j|f(x_j;\hat{\theta}(\boldsymbol{1}))))\|$, $\tilde{g}_{i} \triangleq \|\nabla_{\theta}\ell(z_i, \hat{\theta}(\boldsymbol{1}))\|$ and $g_i = \tilde{g}_i/n$. Then, the approximated Fisher influence function, defined via 
\begin{align}
    \label{eq:Distance_Approx_eq} 
    \tilde{\theta}(w^n) &= \mathrm{Prox}^{\brF(\mD, \hat{\theta}(\boldsymbol{1}))}_{\lambda\pi}(\tilde{\theta}^{\mathrm{IJ,AF}}(w^n)), \quad \text{for } w^n \in \mathcal{D}^{-i}.
\end{align}
satisfies
\begin{align}
    \label{eq:DistBound_Lem}
    \|\tilde{\theta}(w^n) - \hat{\theta}(w^n)\| \leq \frac{2Q C^{2}_f \tilde{g}_{i}}{n^{2}\mu^{2}} + \frac{M\tilde{g}_{i}^{2}}{n^2\mu^3} + \frac{2\tilde{g}_i\tilde{C}_f\bar{E}_n}{n\mu^{2}}, \quad \text{for } w^n \in \mathcal{D}^{-i}.
\end{align}
\end{lem}
\noindent {\em Proof sketch.}
The proof separately bounds the distances between (i) $\hat{\theta}(w^n)$ and $\tilde{\theta}^{\text{IJ}}(w^n)$, and (ii) $\tilde{\theta}^{\text{IJ}}(w^n)$ and $\tilde{\theta}^{\mathrm{IJ,AF}}(w^n)$ for $w^n \in \mathcal{D}^{-i}$. The first bound follows from \citep[Lem.~1]{Wilson_OptimizerComparison}, while the second leverages the closeness of the Hessian and the FIM to show that the estimates remain close. Full proof is provided in \appref{app:Prof_dist}. 
\qed

Similar to prior results \citep{Wilson_OptimizerComparison, WilsonSuriyakumar_UnlearningProximal, Sekhari_Unlearning_Neurips}, the first two terms in \eqref{eq:DistBound_Lem} depend on global problem constants (Lipschitz coefficients, strong convexity parameter, etc.) and the gradient at the $i$th training point. The third term depends on $\bar{E}_n$, which simplifies due to the exponential family structure of the loss and is given by (see \appref{app:exp_family_hessian_proof})
\begin{align}
    \|\nabla_{f}\log(P(y_i|f(x_i;\hat{\theta}(\boldsymbol{1}))))\| = \|t(y_i) - \Esub{\sy\sim P_{\sy|\sx = x_i;\hat{\theta}(\boldsymbol{1})}}{t(\sy)}\|.
\end{align}
Moreover, $\bar{E}_n$ can be shown to serve as an upper bound on the gradient at the optimum $\hat{\theta}(\boldsymbol{1})$ (see \appref{app:exp_family_loss}, \appref{s:Loss_ExpFamily}). In \appref{app:exp_family_loss}, we further demonstrate how this term relates to the absolute training error in classification and regression problems. Specifically, as training error decreases, this term also diminishes. In the extreme case where $\ell(z_i, \hat{\theta}(\boldsymbol{1})) = 0$ for all $i \in [n]$, this term is exactly zero (see \appref{app:Prof_dist}). Thus, we expect the excess term in \eqref{eq:DistBound_Lem} to be small whenever the model's training loss is small. For the remaining terms in \lemref{lem:Distance}, the worst-case discrepancy between $\tilde{\theta}(w^n)$ and $\hat{\theta}(w^n)$ for all $i \in [n]$ is controlled by $
g_{\max} \triangleq \max_{i\in[n]} g_{i}.$
By \assref{ass:boundedmoments} with $(s, r) = (0,1)$, $g_{\max}$ is finite.  

Next, we present our main theorem, which establishes error bounds for the approximated inference objective $T(\cdot)$.

\begin{thm}
    \label{thm:objective_bound}
    Suppose \assref{ass:one_1}, \assref{ass:two_1}, and \assref{ass:LipFeatures}-\assref{ass:LossFunc_ExpFamily} hold. 
    Let $\tilde{\theta}(w^n)$ be defined as in \eqref{eq:Distance_Approx_eq} for $w^n \in \mathcal{D}^{-i}$. Then, 
    \begin{align}
        \label{eq:Objective_bound_final}
        \|T(\hat{\theta}(w^n)) - T(\tilde{\theta}(w^n))\| 
        &\leq C_{T_1} \left( \frac{2Q C^{2}_f \tilde{g}_{i}}{n^{2}\mu^{2}} + \frac{M\tilde{g}_{i}^{2}}{n^{2}\mu^3} + \frac{2\tilde{g}_i\tilde{C}_f\bar{E}_n}{n\mu^{2}} \right) \\
        &\quad + \frac{1}{2} C_{T_2} \left( \frac{2Q C^{2}_f \tilde{g}_{i}}{n^{2}\mu^{2}} + \frac{M\tilde{g}_{i}^{2}}{n^{2}\mu^3} + \frac{2\tilde{g}_i\tilde{C}_f\bar{E}_n}{n\mu^{2}} \right)^{2}
    \end{align}
    and, 
    \begin{align}
        \label{eq:Taylor_Proof}
        \| T(\hat{\theta}(w^n)) &- T(\hat{\theta}(\boldsymbol{1})) 
        - \langle \nabla T(\hat{\theta}(\boldsymbol{1})), \tilde{\theta}(w^n) - \hat{\theta}(\boldsymbol{1}) \rangle \| \\
        &\leq C_{T_1} \left( \frac{2Q C^{2}_f \tilde{g}_{i}}{n^{2}\mu^{2}} + \frac{M\tilde{g}_{i}^{2}}{n^{2}\mu^3} + \frac{2\tilde{g}_i\tilde{C}_f\bar{E}_n}{n\mu^{2}} \right) + \frac{2C_{T_2}\tilde{g}^{2}_i}{n^{2}\mu^{2}}.
    \end{align}
\end{thm}
\noindent {\em Proof sketch.}
    Both bounds follow from the smoothness properties of \( T \) (\assref{ass:Lipschitz_T}), combined with \lemref{lem:Distance} and \lemref{lem:gen_strong_convex} from \appref{app:Closeness_Weighted}. Full proof is provided in \appref{app:InferenceObjective_Bound}.  \qed

\thmref{thm:objective_bound} enables a systematic derivation of theoretical guarantees for FIM-based influence approximations across various application areas. Moreover, as discussed in \citep[Sec.~3]{giordano2019swiss}, for weight vectors \( w^n = \mathcal{D}^{-i} \), we expect \( \lim_{n\to\infty} g_{\max} = 0 \). Consequently, whenever \( \bar{E}_n \to 0 \), \thmref{thm:objective_bound} ensures that \( T(\tilde{\theta}(w^n)) \) and the Taylor-series approximation (Equation~\eqref{eq:TaylorSeries_Objective} with \( w^n \in \mathcal{D}^{-i} \)) converge to \( T(\hat{\theta}(w^n)) \) for all \( i \in [n] \). However, as we demonstrate in \secref{s:Exp_Fairness} and \secref{s:Exp_CV}, in practice, \( \tilde{\theta}(w^n) \) is often a good approximation of \( \hat{\theta}(w^n) \) even when \( \bar{E}_n \) is finite.


Next, we show that our framework provides guarantees in a unified manner, analogous to \propref{prop:CV_Guarantees_Wilson}–\propref{prop:attributions}, which establish Hessian-based guarantees for several tasks outlined in \secref{s:ProblemDef}.

\begin{corol}[LOOCV]
\label{corol:LOOCV_FIM}
Suppose \assref{ass:one_1}, \assref{ass:two_1}, and \assref{ass:boundedmoments}-\assref{ass:LossFunc_ExpFamily} hold with $(s,r) = \{(0, 2), (0, 3), (1, 2), (1, 3), (1, 4)\}$. Let \( T(\theta, \boldsymbol{1}^{n\backslash i}) = \frac{1}{n} \ell(z_i, \theta) \triangleq T_i(\theta) \). When \( \tilde{\theta}(w^n) \) from \lemref{lem:Distance} is used as a plug-in estimate for \( w^n \in \mathcal{D}^{-i} \), the error in the approximate cross-validation estimate satisfies:
\begin{align}
    \left| \sum_{i=1}^n \left( T_i(\tilde{\theta}(\boldsymbol{1}^{n\backslash i})) - T_i(\hat{\theta}(\boldsymbol{1}^{n\backslash i})) \right) \right| 
    \leq O\left( \frac{MB_{03}}{\mu^3 n^2} + \frac{C^{2}_f B_{02}}{\mu^{2} n^{2}} + \frac{\tilde{C}_f \bar{E}_n B_{02}}{\mu^{2} n} \right).
\end{align}
\end{corol}

\begin{corol}[Machine Unlearning]
\label{corol:unlearning_FIM}
Suppose \assref{ass:one_1}, \assref{ass:two_1}, \assref{ass:LipFeatures}, and \assref{ass:LossFunc_ExpFamily} hold. Assume further that \( \tilde{g}_i \leq G \) for all \( i \in [n] \). Then, for the inference objective \( T(\theta) = \theta \), we have:
\begin{align}
    \|T(\tilde{\theta}(w^n)) - T(\hat{\theta}(w^n))\| 
    \leq \frac{2Q C^{2}_f G}{n^{2}\mu^{2}} + \frac{MG^{2}}{n^{2}\mu^3} + \frac{2G\tilde{C}_f\bar{E}_n}{n\mu^{2}}, \quad \text{for } w^n \in \mathcal{D}^{-i}.
\end{align}
Furthermore, the algorithm returning \( \tilde{\theta}(w^n) + \zeta \) satisfies \( (\eps, \delta) \)-unlearning, where \( \zeta \sim \pazocal{N}(0 ,c\brI) \) and:
\begin{align}
    c = \left( \frac{2Q C^{2}_f G}{n^{2}\mu^{2}} + \frac{MG^{2}}{n^{2}\mu^3} + \frac{2G\tilde{C}_f\bar{E}_n}{n\mu^2} \right) \frac{\sqrt{2\log(5/4\delta)}}{\eps}.
\end{align}
\end{corol}

\begin{corol}[Data Attribution] 
\label{corol:Attributed_FIM}
    Suppose the assumptions of \thmref{thm:objective_bound} hold, $T(\theta) = \ell(z_{\mathrm{test}}, \theta) - \ell(z_{\mathrm{test}}, \hat\theta(\boldsymbol{1}))$ and $C_{\ell} \triangleq \underset{i\in[n]}{\max} \ \tilde{g}_i$. Then,  
    \begin{align}
        &|T(\hat{\theta}(\boldsymbol{1}^{n\backslash i})) - T(\tilde{\theta}(\boldsymbol{1}^{n\backslash i}))| \leq O\left(\frac{C^2_f C_{T_1}C_{\ell}}{n^2\mu^2} + \frac{M C_{T_1}C^{2}_{\ell}}{n^2\mu^3} + \frac{C_{T_1}\tilde{C}_f\bar{E}_nC_{\ell}}{n\mu^2}\right), \\ 
           &| T(\hat{\theta}(\boldsymbol{1}^{n\backslash i})) - T(\hat{\theta}(\boldsymbol{1})) - \langle \nabla T(\hat{\theta}(\boldsymbol{1})),\tilde{\theta}(\boldsymbol{1}^{n\backslash i}) -\hat{\theta}(\boldsymbol{1})\rangle | \notag\\ 
         &\qquad\qquad \leq O\left(\frac{C^2_f C_{T_1}C_{\ell}}{n^2\mu^2}  + \frac{M C_{T_1}C^{2}_{\ell}}{n^2\mu^3}  + \frac{C_{T_2}C^{2}_{\ell}}{n^2\mu^2}  + \frac{C_{T_1}\tilde{C}_f\bar{E}_nC_{\ell}}{n\mu^2}\right)
    \end{align}
\end{corol}

The proofs for these corollaries rely on applying \thmref{thm:objective_bound} for the settings described in \propref{prop:CV_Guarantees_Wilson} - \propref{prop:attributions} (see \appref{app:corol_proofs}). To further demonstrate the generality of our approach, we provide guarantees for the fairness assessment task described in \secref{s:ProblemDef}, for which currently there is no theoretical analysis. The proof is in \appref{app:corol_Fairness}. 
\begin{corol}[Fairness Evaluation]
    \label{corol:fairness}
    Suppose \assref{ass:one_1}, \assref{ass:two_1}, and \assref{ass:LipFeatures} -\assref{ass:LossFunc_ExpFamily} hold. If $T$ be given by \eqref{eq:DP_Def} and $C_{\ell} \triangleq \underset{i\in[n]}{\max} \ \tilde{g}_i$. Then, 
    \vspace{0pt}
    \begin{align}
        &|T(\hat{\theta}(\boldsymbol{1}^{n\backslash i})) - T(\tilde{\theta}(\boldsymbol{1}^{n\backslash i}))| \leq O\left(\frac{C^3_fC_{\ell}}{n^2\mu^2} + \frac{M C_fC^{2}_{\ell}}{n^2\mu^3} + \frac{C_f\tilde{C}_fC_{\ell}\bar{E}_n}{n\mu^2}\right).
    \end{align}
\end{corol}
To the best of our knowledge, \corolref{corol:LOOCV_FIM} - \corolref{corol:fairness} provide the first theoretical guarantees for using the FIM in influence assessment tasks, offering a novel method with rigorous effectiveness proof. Additionally, our framework easily extends to other problems in machine learning and statistics beyond the specific applications discussed (e.g., data dropping \citep{Broderick_Robustness}).

\begin{remark}[The Non-Convex Setting]
    While many theoretical analyses of influence (e.g., \citep{KohLiang_Influence_DL}) assume a convex, differentiable loss, these assumptions often do not hold in practice. Nonetheless, influence functions remain widely used for influence assessment \citep{KohLiang_Influence_DL, han2020explaining}. Recent work \citep{Bae_PBRF_Influence} shows that a variant of Fisher influence corresponds to the minimizer of an approximation to the Proximal Bregman Response Function (PBRF). This finding helps explain the empirical usefulness of influence functions in more complex domains and illustrates how analyses grounded in convex assumptions can still offer valuable insights for non-convex scenarios. Our probabilistic framework extends these results by introducing $\tilde{\theta}(\boldsymbol{1}^{n\backslash i})$, which depends on $\tilde{\theta}^{\mathrm{IJ,AF}}(\boldsymbol{1}^{n\backslash i})$ and can be computed efficiently. It further provides a theoretical justification for using the AFIF by establishing bounds on $\|\tilde{\theta}^{\mathrm{IJ}}(\boldsymbol{1}^{n\backslash i}) - \tilde{\theta}(\boldsymbol{1}^{n\backslash i})\|$ under mild assumptions likely to hold locally (see \appref{app:dist_bound_IJ_F}). These results support adopting AFIF over traditional Hessian-based methods. Moreover, while \citep{Bae_PBRF_Influence} focuses on $\pi(\theta) = \|\theta\|^2$, our framework readily accommodates non-differentiable regularizers. Since training models with general regularization (beyond $L_2$) is an increasingly popular method for adding robustness, feature sparsity, and interoperability to models (see \cite{lemhadri2021lassonet, li2021ell} and references therein), our approach gives a state-of-the-art tool for quantifying influence in these cases.
\end{remark}
\section{Experiments}
\label{s:Exps}
We evaluate the utility of approximate Fisher influence through experiments on three different tasks. Both Fisher-based and Hessian-based influence functions are implemented within the same codebase, differing only in the automatic differentiation components used to compute \eqref{eq:Esimtates_FIM} and \eqref{eq:Stochastic_Hess}. Detailed experimental procedures are provided in \appref{app:ExpDetails}. Our objective is to demonstrate the advantages of AFIF across different tasks by showing that, across a set of different classical influence measurement settings, it:
\begin{enumerate}
    \item Achieves similar utility as the Hessian-based techniques
    \item Has improved computational efficiency relative to the Hessian-based techniques. 
\end{enumerate}
We will further demonstrate the usefulness of our technique in a setting that involves a non-differentiable regularizer, demonstrating a novel method for measuring influence in these cases. The codebase to reproduce our experimental results is provided in \url{https://github.com/omrilev1/Approximate-Fisher-Influence}. 
\subsection{Fairness and Unlearning}
\label{s:Exp_Fairness} 

\begin{figure}[t]
    \centering
    \begin{subfigure}[t]{0.32\textwidth}
        \centering
        \includegraphics[width=\textwidth]{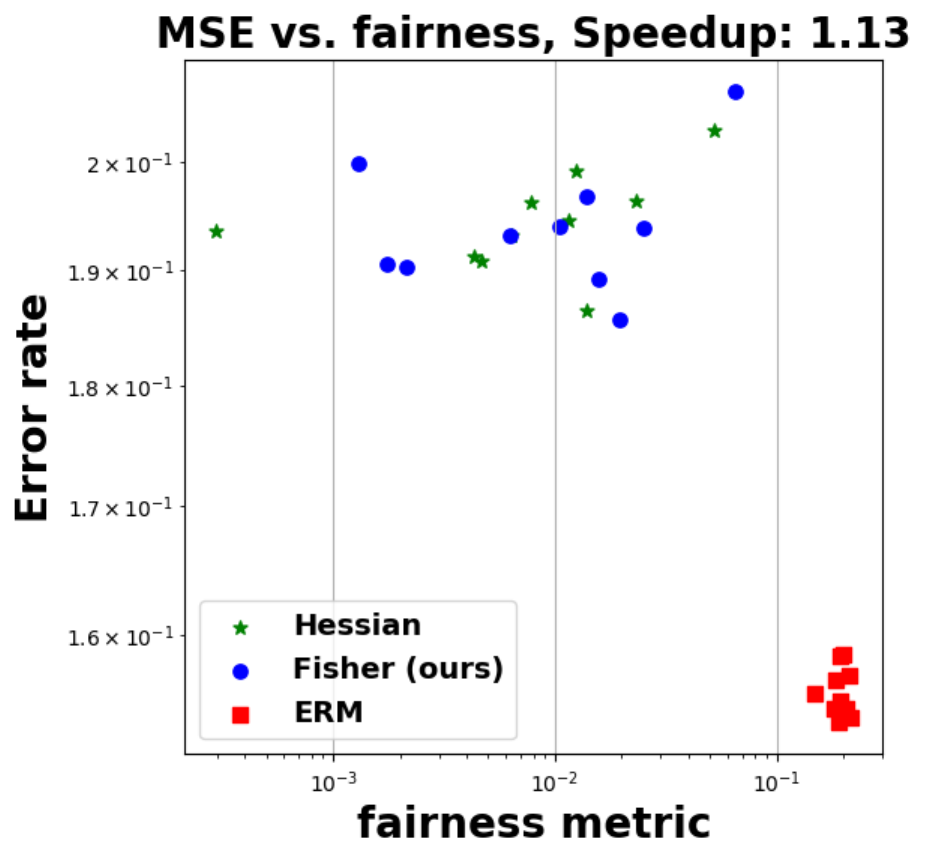}
        \caption{Adult}
        \label{fig:adult}
    \end{subfigure}
    \hfill
    \begin{subfigure}[t]{0.32\textwidth}
        \centering
        \includegraphics[width=\textwidth]{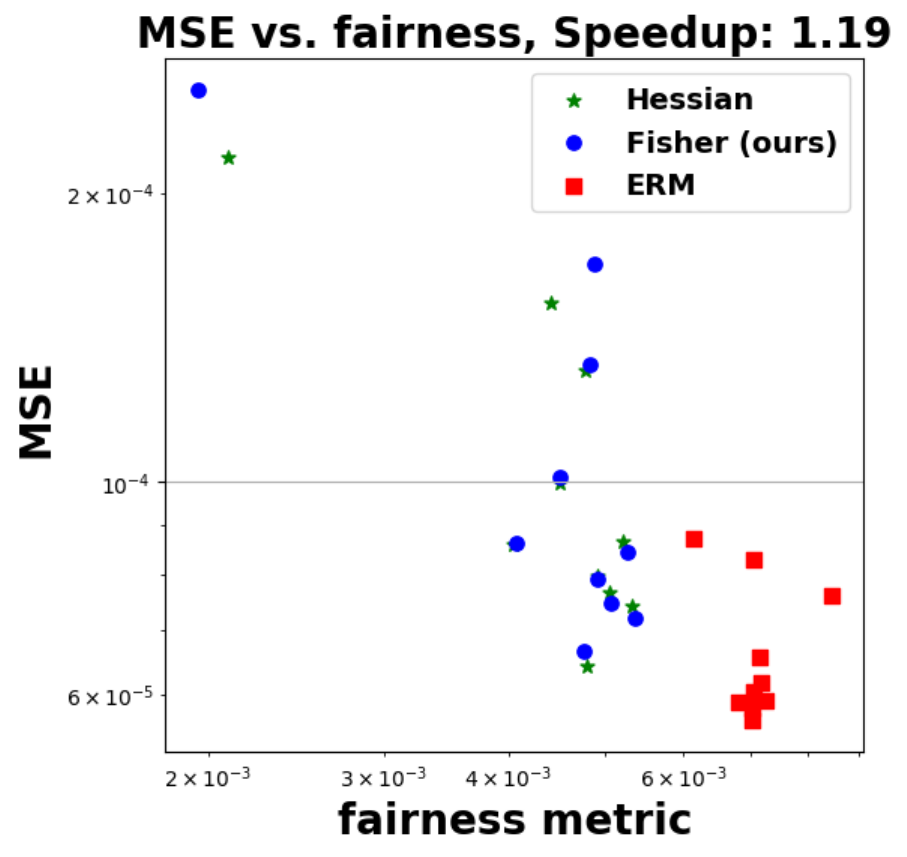}
        \caption{Insurance}
        \label{fig:insurance}
    \end{subfigure}
    \hfill
    \begin{subfigure}[t]{0.32\textwidth}
        \centering
        \includegraphics[width=\textwidth]{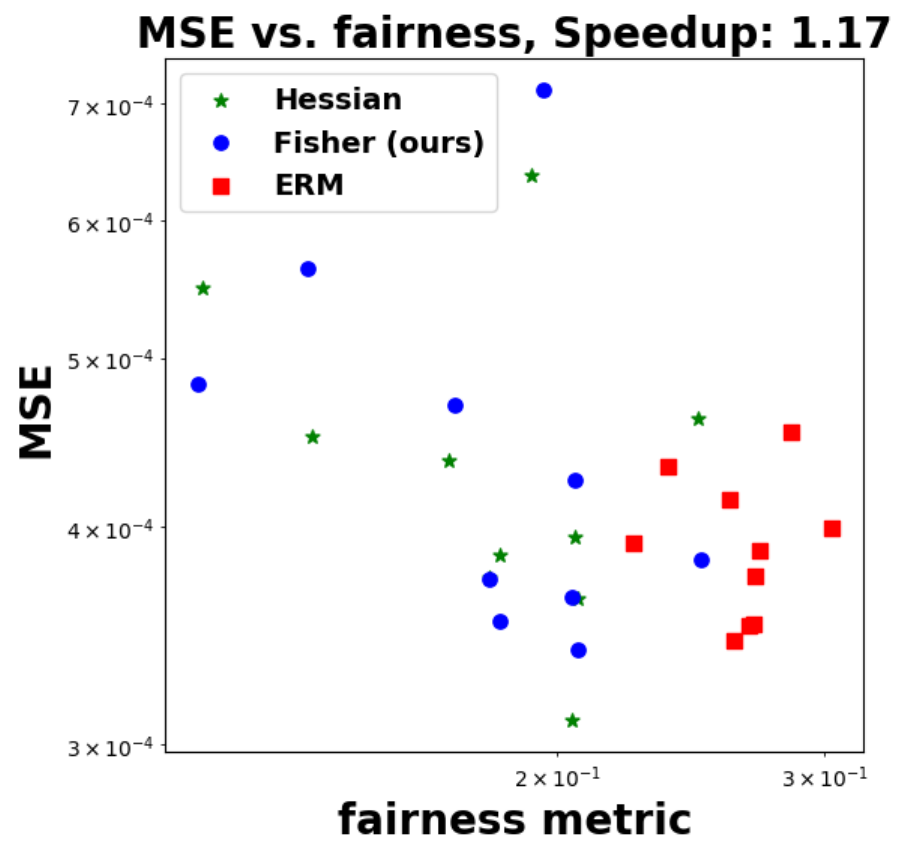}
        \caption{Crime}
        \label{fig:crime}
    \end{subfigure}
    
    \caption{Model performance versus fairness metric for Fisher-based influence, Hessian-based influence, and the ERM solution from \eqref{eq:WERM_Def}, evaluated on the Adult, Crime, and Insurance datasets using a two-layer classifier. Results are averaged over ten independent experiments. All cases demonstrate that the Fisher-based computations are faster than the Hessian-based computations yet still yield similar overall utility.}
    \label{fig:Fairness}
\end{figure}

In this set of experiments, we aim to identify and unlearn training points that negatively impact model fairness. To that end, we use three classical datasets from the fairness literature \citep{shah2022selective, AbhinShah_Fairness, sattigeri2022fair}: Adult, Crime, and Insurance datasets. For the Adult dataset, the goal is to classify whether a person's income is greater than $50,000\$$, while keeping the classifier independent of the person's sex. For the crime dataset, the goal is to predict crime per population (which is a continuous variable), while keeping the regressor independent of race. For the insurance dataset, the goal is to predict medical expenses and make the regressor independent of sex. Fairness is assessed using the demographic parity metric for the adult dataset and using the $\chi^2$ divergence for the crime and the insurance datasets. Our models are two-layer networks with \texttt{SeLU} activations, similar to the architectures from \cite{ghosh2023influence, AbhinShah_Fairness, sattigeri2022fair}. We used \eqref{eq:DP_Def} and \eqref{eq:chi2_def} as our inference objectives and calculated the influence for each training sample using the plug-in estimator from \thmref{thm:objective_bound}. We then unlearned all training samples with positive influence by applying \eqref{eq:Distance_Approx_eq}. Full experimental details are provided in \appref{app:ExpDetails}.

We measured the time required to compute influences and unlearn samples using both Fisher-based and Hessian-based calculations, reporting the model's performance (measured via classification accuracy for the adult dataset and MSE for the crime and insurance datasets) and estimated fairness (measured either via DP or via the $\chi^2$ measure) after data removal. As shown in \figref{fig:Fairness}, both methods perform similarly, significantly improving fairness score without substantial performance loss, matching results from \cite{sattigeri2022fair, AbhinShah_Fairness}. However, the Fisher-based results are consistently faster relative to the Hessian-based approaches, demonstrating the computational efficiency of the Fisher-based influence. We give details of additional experiments in \appref{app:AdditionalExps_Fairness}, which show that the Hessian-based influence fails to improve the model's fairness and to maintain the same performance for different choices of hyperparameters, demonstrating potential instabilities without proper hyperparameter tuning. Additionally, the error rates and MSE of the ERM minimizers are strictly positive, corresponding to a finite $\bar{E}_n$. Nevertheless, the AFIF effectively identifies and unlearns samples that negatively impact fairness, demonstrating its usefulness when $\bar{E}_n$ is finite.

\subsection{Cross-Validation}
\label{s:Exp_CV}  
In our second example, we establish our method's computational advantage and demonstrate the improved stability of the approximate Fisher influence, as described in our prior remark about computational stability, when used to approximate cross-validation. To that end, we used the same two-layer model used in \secref{s:Exp_Fairness} for the adult dataset, increased the width of the hidden layer to $30000$, and have trained the model with a weight decay of $10^{-8}$. We thus expect the model's Hessian to be ill-conditioned, preventing \eqref{eq:G_Definition} from working without a proper regularization. Our goal was to estimate the test loss of the model as a function of the number of epochs using CV. To the best of our knowledge, this is the first work to apply the FIM to approximate CV. To reduce the computational complexity of the LOOCV, we used a leave-$k$-out CV with k set to $6000$, corresponding to $\sim$20\% of the trainset, and then averaged five different estimates to generate the final value (see \appref{app:CV_Details} for further details). \figref{fig:Exp_CV} reports the test loss, estimated loss, and average computation time (to generate an estimate based on the five different folds) for each method. The results show that the Hessian-based method fails to converge to the correct estimate, while the Fisher-based method follows the test loss trend, demonstrating potential instabilities of using \eqref{eq:G_Definition}. Additional experiments in \appref{app:ExpDetails_CV} confirm this behavior across other hyperparameter choices. Moreover, Fisher-based CV requires $\sim$50\% less time than the Hessian-based estimate. 

\begin{figure}
    \centering
    \includegraphics[width=0.5\columnwidth]{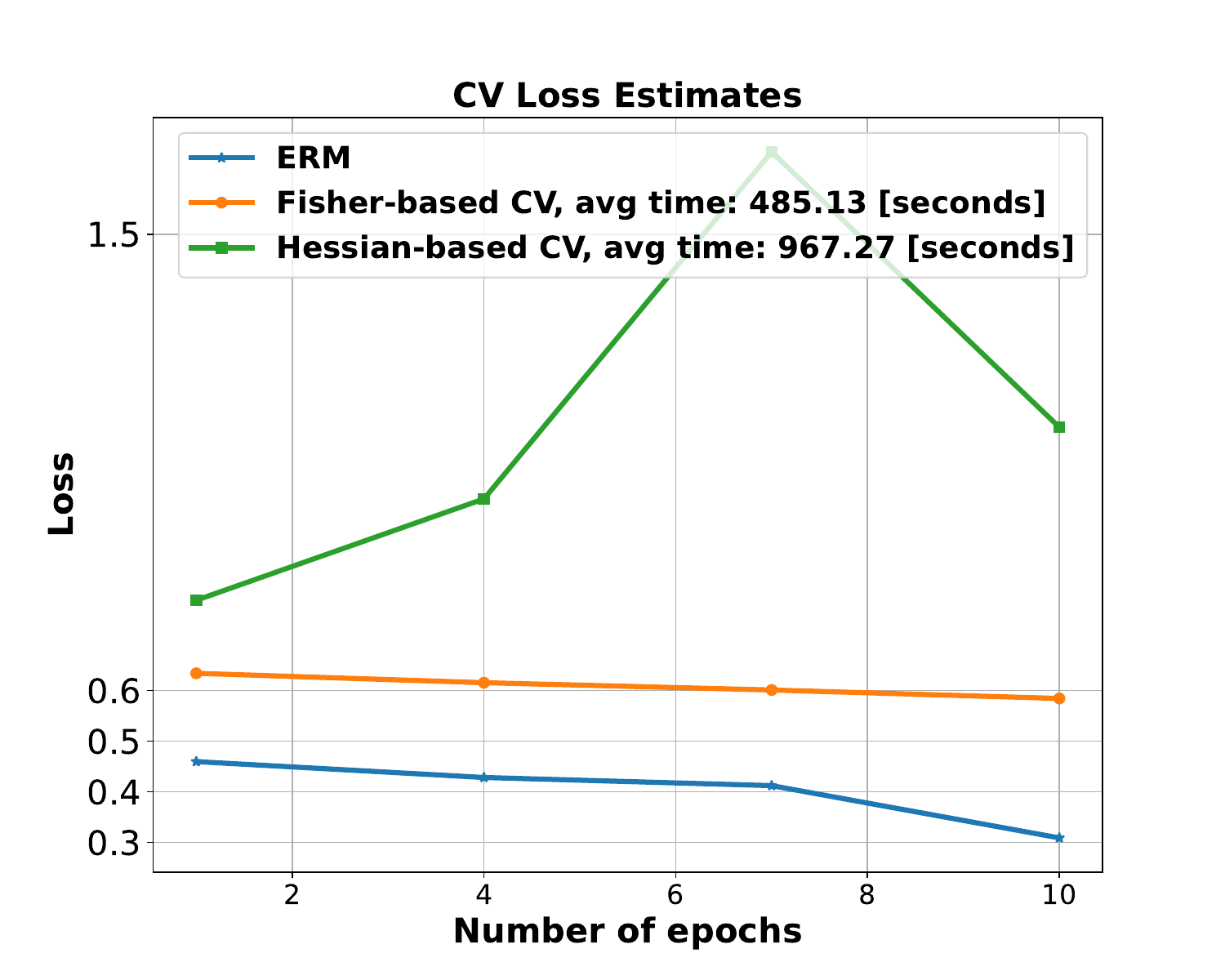}
    \caption{Test loss and CV estimators for Fisher and Hessian-based influence on the Adult dataset using a two-layer classifier, averaged over five folds. Fisher calculations are approximately twice as fast as Hessian computations and Hessian estimates are highly unstable, yielding invalid loss estimates.}
    \label{fig:Exp_CV}
\end{figure}

\subsection{Data Attribution}
\label{s:Exp_Influence}

To demonstrate the effectiveness of the AFIF in a high-dimensional non-convex setting, we attribute test sample predictions to training data using two popular neural network architectures: the ResNet18 \citep{he2016deep} and a model comprised of three convolutional layers and two fully connected layers, on a subset of CIFAR-10 \citep{CIFAR10_Krizhevsky}, focusing on the ``plane" and ``car" classes (see \appref{app:ExpDetails}). We calculated the influence of training examples on the $30$ test instances with the highest test loss. \figref{fig:main_figure_CNN}, \figref{fig:main_figure_ResNet} and \figref{fig:main_figure_running_times} present the three images with the highest and the lowest influence scores and the computation times for both cases. Both methods identified the same influential training samples, with a maximal discrepancy between influence scores, which was less than $20\%$ of the maximal influence value. However, the AFIF calculations were faster than the Hessian-based calculations in both simulated cases.  

\begin{figure*}[tt]
    \centering
    \begin{subfigure}{0.45\textwidth}
        \centering
        \includegraphics[width=\textwidth]{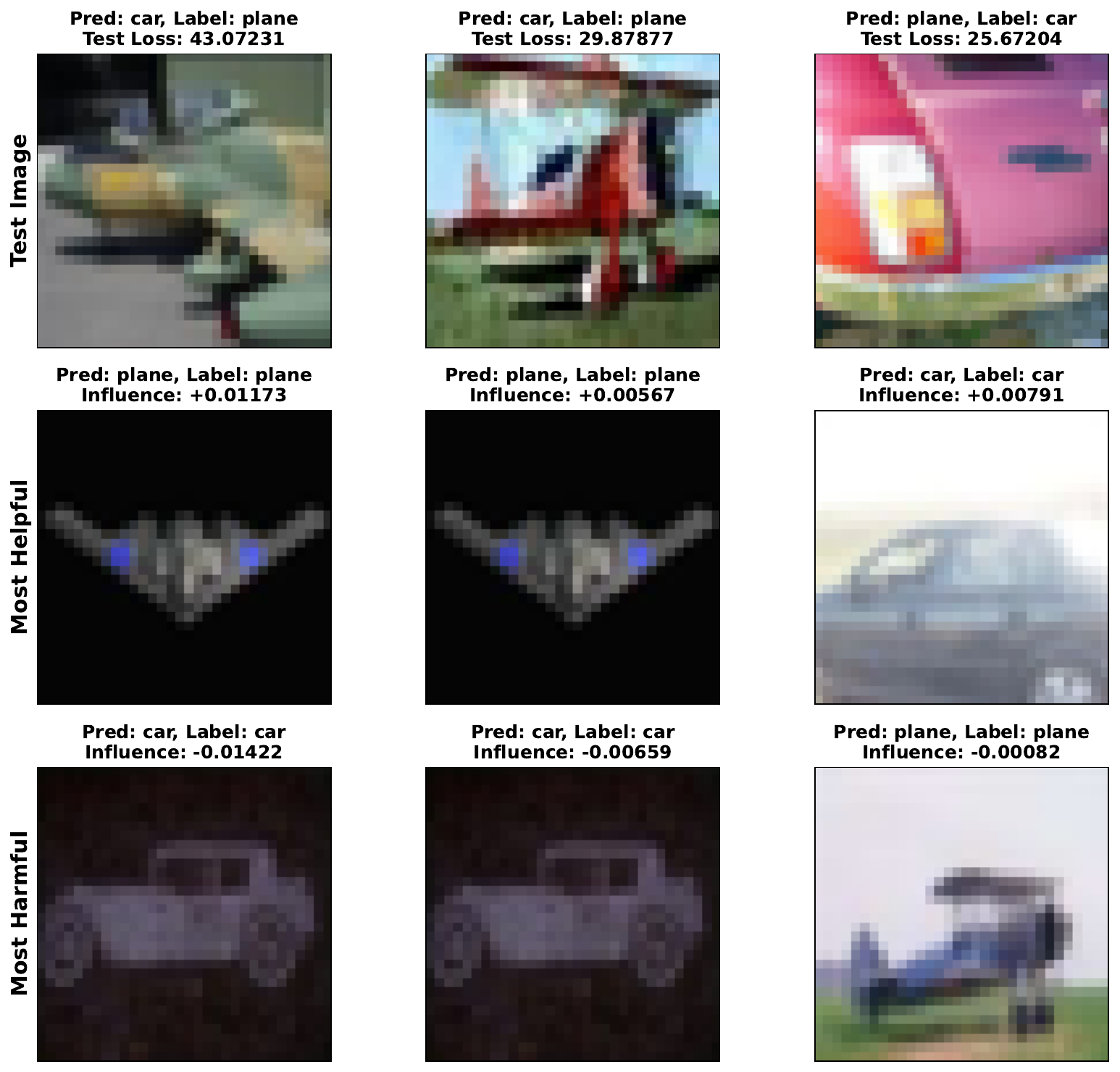}
        \caption{Fisher}
    \end{subfigure}
    \hfill
    \begin{subfigure}{0.45\textwidth}
        \centering
        \includegraphics[width=\textwidth]{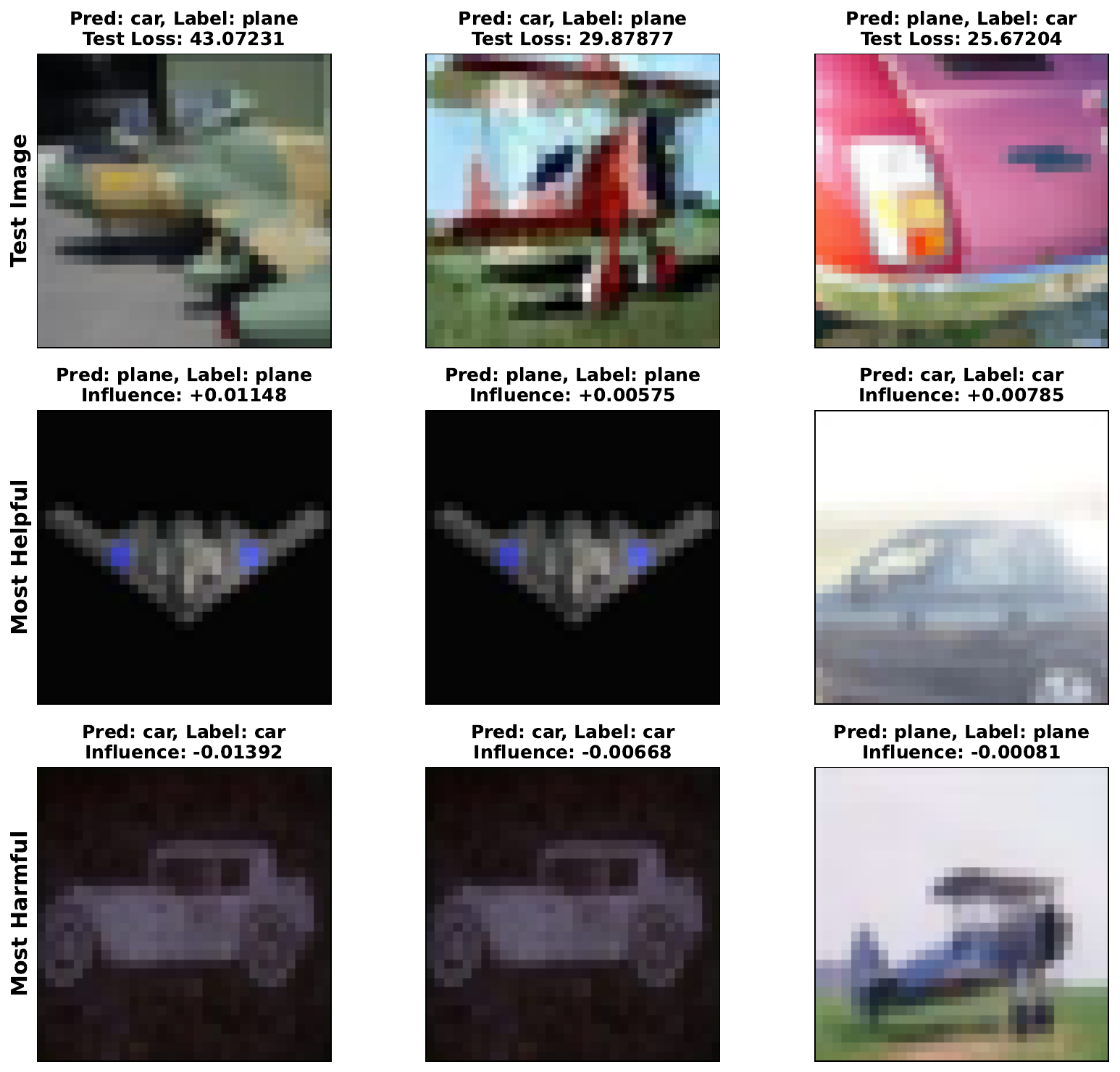}
        \caption{Hessian}
    \end{subfigure}

    \caption{Most and least influential images on a subset of CIFAR10 when using a simple CNN architecture.}
    \label{fig:main_figure_CNN}
\end{figure*}

\begin{figure*}[tt]
    \centering

    \begin{subfigure}{0.485\textwidth}
        \centering
        \includegraphics[width=\textwidth]{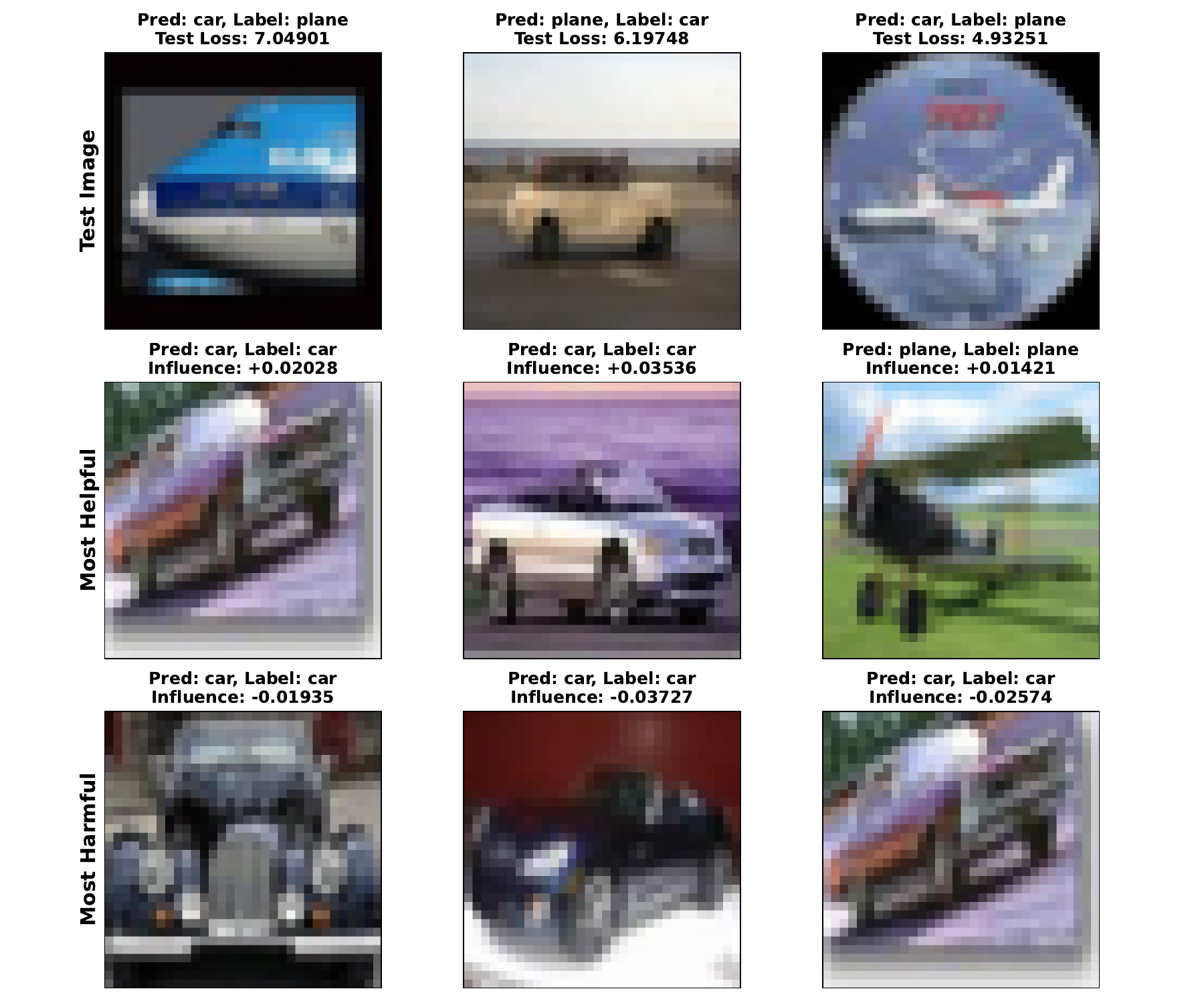}
        \caption{Fisher}
    \end{subfigure}
    \hfill
    \begin{subfigure}{0.485\textwidth}
        \centering
        \includegraphics[width=\textwidth]{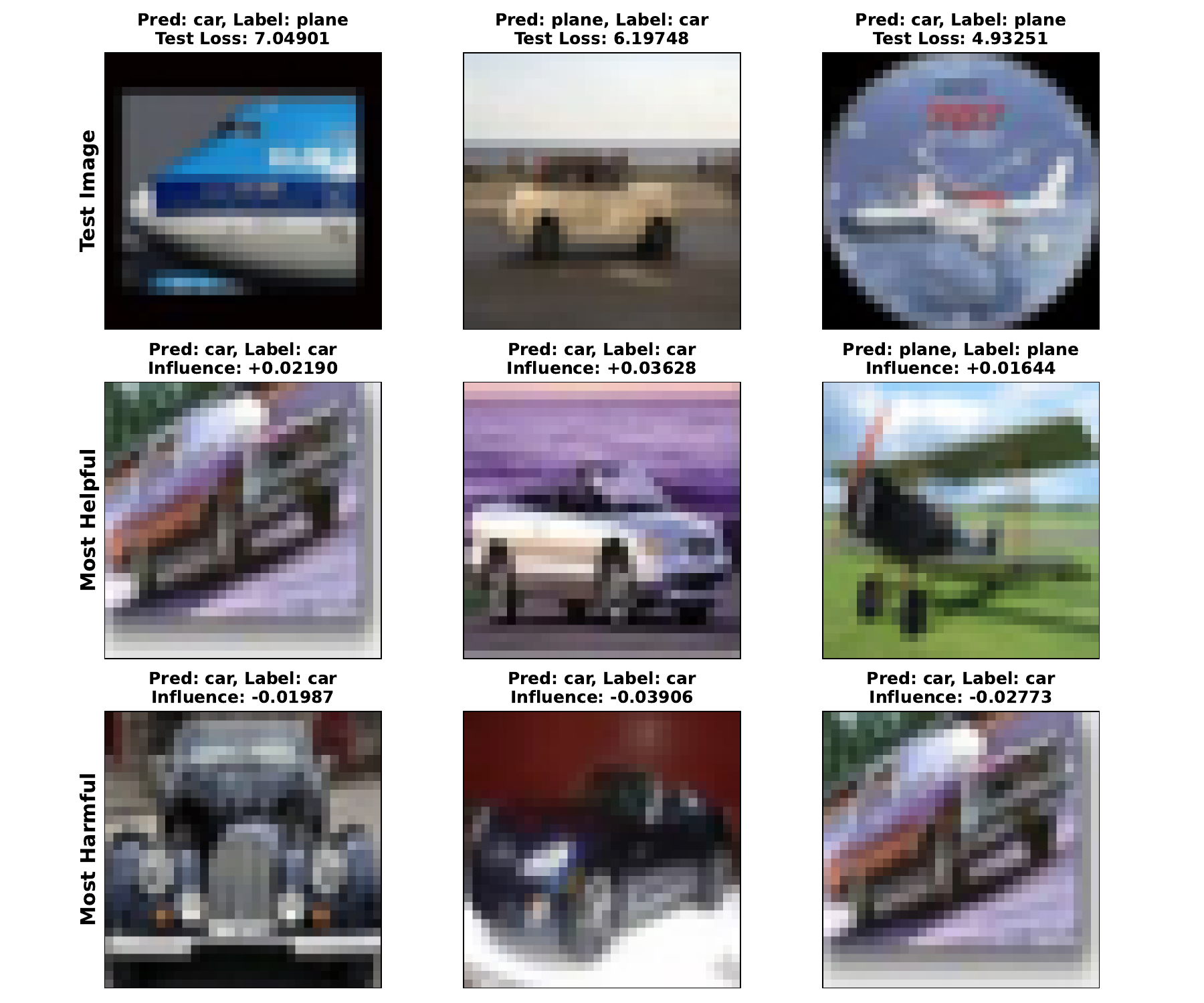}
        \caption{Hessian}
    \end{subfigure}

    \caption{Most and least influential images on a subset of CIFAR10 when using the ResNet18 architecture. }
    \label{fig:main_figure_ResNet}
\end{figure*}

\begin{figure*}
    \centering
    \begin{subfigure}{0.44\textwidth}
        \centering
        \includegraphics[width=\textwidth]{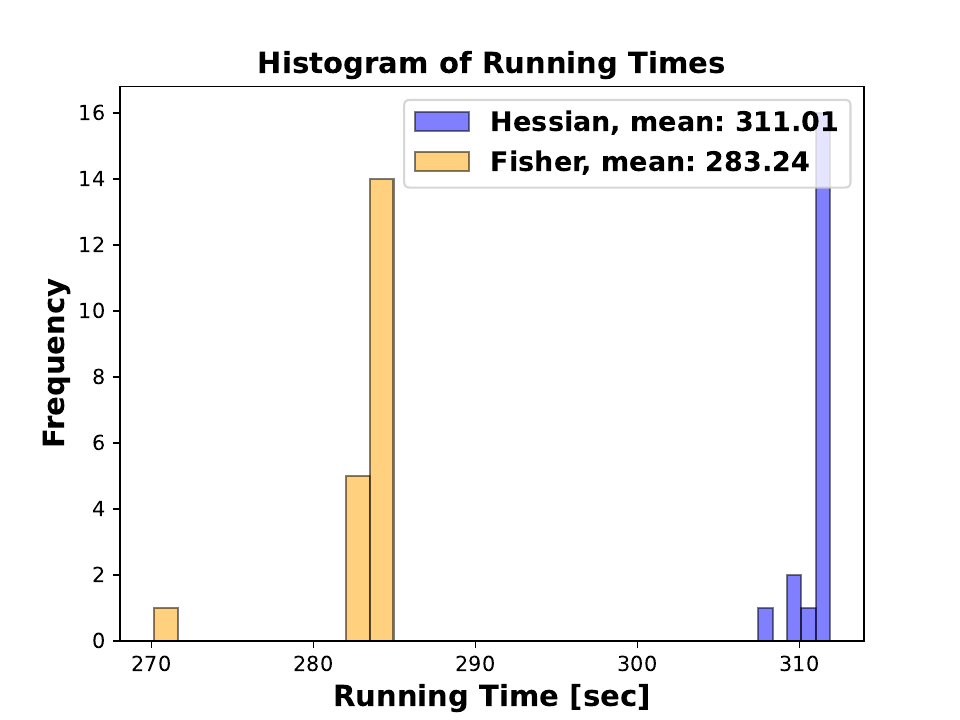}
        \caption{CNN}
    \end{subfigure}
    \hfill
    \begin{subfigure}{0.475\textwidth}
        \centering
        \includegraphics[width=\textwidth]{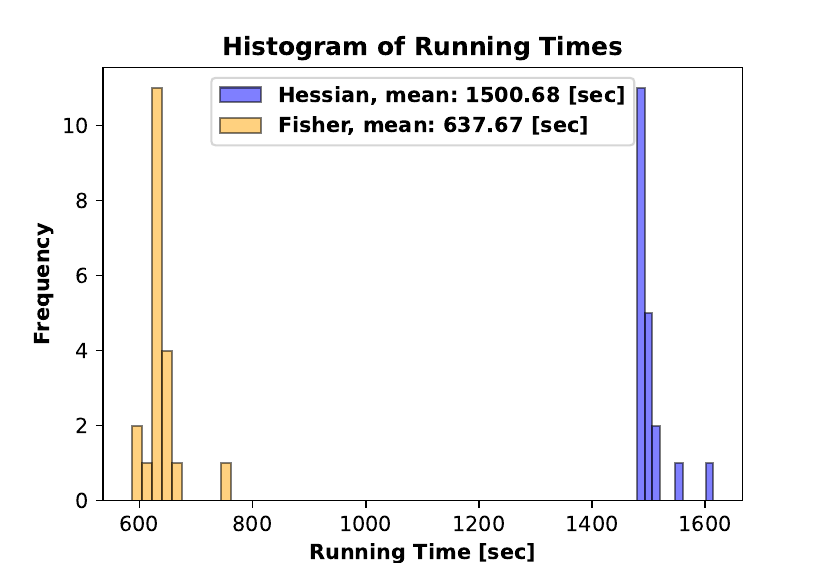}
        \caption{ResNet 18}
    \end{subfigure}

    \caption{Running times for Fisher-based and Hessian-based influence function when calculated on a subset of CIFAR10 classified using ResNet18 and a simple three-layer CNN. In both cases, Fisher-based influence significantly accelerates the influence calculation.}
    \label{fig:main_figure_running_times}
\end{figure*}

\section{Concluding Remarks}
In this work, we introduced the AFIF, a novel method for quantifying influence in machine learning models. By using the FIM instead of the Hessian, we demonstrate how our technique is fundamentally faster than existing influence function baselines yet provides similar error guarantees across a set of tasks. Moreover, our framework extends the applicability of influence measurement to a broader range of scenarios—including those involving non-differentiable regularizers. We demonstrated the computational efficiency of AFIF relative to traditional Hessian-based techniques and its usefulness in providing reliable influence estimates across a set of tasks in a set of empirical evaluations.

Generalizing our analysis to more complex, real-world influence measurement methods that are based on the FIM and currently lack rigorous theoretical support (for example, techniques based on the Kronecker-Factored FIM \cite{choe2024_data_GPT}) is a promising future research direction, that will open the door to systematically determining when and how such methods can be most effectively applied across diverse tasks. Moreover, developing computationally efficient variants of higher-order influence measurement techniques such as those explored in \cite{giordano2019higher, basu2020second} (see also the discussion in \cite{KohLiang_Influence_DL}) by utilizing the underlying statistical nature of the optimization problem is another future research direction, that is currently under investigation.

\bibliographystyle{plainnat} 

\appendix

\clearpage 
\renewcommand{\contentsname}{Contents of Appendix}
\tableofcontents
\addtocontents{toc}{\protect\setcounter{tocdepth}{3}}  

\section{Definitions and Useful Lemmas}
\label{app:Defs}
The manuscript uses the next classical definitions from the convex optimization theory \citep{NesterovBook}. 

\begin{defn}[Matrix Operator-Norm]
    For any matrix $\brA$ we define its \emph{operator-norm} by $\norm{\brA}_{\text{op}} \triangleq \underset{v\in\reals^{d}: \norm{v}\ne 0}{\sup} \ \norm{\brA v}/\norm{v}$ 
\end{defn}

\begin{defn}[Strong convexity]
    Let $\beta > 0$. A function $f(\cdot)$ is $\beta$-strongly convex if and only if
    \begin{align}
        f(y) \geq f(x) + \nabla^{\top}f(x) (y - x) + \frac{\beta}{2} \norm{x - y}^2, \ \forall (x,y)\in \text{dom}(f) 
    \end{align}
\end{defn}

\begin{defn}[Lipschitz]
    A function $f(\cdot)$ is $C$-Lipschitz if
    \begin{align}
        \norm{f(x) - f(y)} \leq C \norm{x - y}, \ \forall (x,y) \in \text{dom}(f).
    \end{align}
    In that case, $C$ is called the Lipschitz constant of $f$ and is denoted by $C \triangleq \mathrm{Lip}(f(x))$. 
\end{defn}

\begin{defn}[Smooth]
    If $f(\cdot)$ is differentiable, then $f(\cdot)$ is $K$-smooth if 
    \begin{align}
        \norm{\nabla f(x) - \nabla f(y)} \leq K\norm{x - y}, \ \forall (x,y) \in \text{dom}(f).
    \end{align}
    In that case, $K$ is called the gradient-Lipschitz constant of $f$ and is denoted by $C \triangleq \mathrm{Lip}_1(f(x))$. 
\end{defn}

\begin{defn}[Lipschitz-Hessian]
    If $f(\cdot)$ is twice differentiable, then $f(\cdot)$ is $M$-Lipschitz Hessian if 
    \begin{align}
        \norm{\nabla^{2}f(x) - \nabla^{2}f(y)}_{\mathrm{op}} \leq M\norm{x - y}, \ \forall (x,y) \in \text{dom}(f)
    \end{align}
    In that case, $M$ is called the Lipschitz-Hessian constant of $f$ and is denoted by $M \triangleq \mathrm{Lip}_2(f(x))$. 
\end{defn}

Throughout the manuscript, we will further make use of the next connections between Lipschitz coefficients and gradient bounds for differentiable functions.
\begin{corol}[\cite{NesterovBook}]
    Let $f(x)$ be a differentiable function $\forall x\in\text{dom}(f)$. Then, $f(x)$ is $C$-Lipschitz if and only if 
    \begin{align}
        \|\nabla f(x)\| \leq C, \ \forall x\in\text{dom}(f). 
    \end{align}
    If $f(x)$ is twice-differentiable $\forall x\in \text{dom}(f)$ then $f(x)$ is $K$-smooth if and only if  
    \begin{align}
        \|\nabla^{2} f(x)\|_{\mathrm{op}} \leq K, \forall x\in\text{dom}(f).    
    \end{align} 
\end{corol}

\section{Example for Losses From an Exponential Family}
\label{app:exp_family_loss}
We now present a few examples of commonly used loss functions in machine learning that can be viewed as the negative log-likelihood of an exponential-family model. Specifically, let
\[
  \ell\bigl(y, f(x;\theta)\bigr) 
  =
  -\log P\bigl(y \mid f(x;\theta)\bigr),
\]
where \(P(y \mid f(x;\theta))\) belongs to a (discrete) exponential family. Throughout this paper, we adopt the following form of an exponential family:
\begin{align}
    \label{eq:General_Exp_Family}
    \log P\bigl(y \mid f(x;\theta)\bigr)
    =
    f(x;\theta)^\top\,t(y)
    -
    \log\biggl(\sum_{\tilde{y}=1}^{|\pazocal{Y}|}
    \exp\bigl\{\,f(x;\theta)^\top\,t(\tilde{y})\bigr\}\biggr)
    \;+\;\beta(y),
\end{align}
where \(t(y)\) are the \emph{natural statistics} and \(f(x;\theta)\) are the \emph{natural parameters}.\footnote{The above is a discrete version; for continuous \(\pazocal{Y}\), one replaces the sum with an integral.} 
The term \(\beta(y)\) depends only on \(y\) (thus does not affect parameter learning) and ensures proper normalization. Below, we illustrate two popular examples of loss functions (see also \cite[Sec.~9.2]{NGD_Review}) that arise naturally from this exponential-family framework.

\begin{enumerate}
\item \textbf{Cross-Entropy Loss.}
A standard approach in multi-class classification over \(\abs{\pazocal{Y}}\) classes is the softmax parameterization:
\begin{align}
    \label{eq:loss_classification}
    \log P\bigl(y \mid f(x;\theta)\bigr)
  =
  \bigl(f(x;\theta)\bigr)_{y}
  -
  \log \Bigl(\sum_{\tilde{y}=1}^{|\pazocal{Y}|}
  \exp\{\bigl(f(x;\theta)\bigr)_{\tilde{y}} \}\Bigr),
  \quad
  y,\tilde{y} \in \{1,\dots,\abs{\pazocal{Y}}\}.
\end{align}

Here, \(f(x;\theta)\) is a vector of length \(\abs{\pazocal{Y}}\). By defining \(e_y\) as the one-hot vector with a 1 in the \(y\)-th entry and 0 elsewhere, we see that 
\[
  \log P\bigl(y \mid f(x;\theta)\bigr)
  =
  f(x;\theta)^\top e_{y}
  -
  \log \Bigl(\sum_{\tilde{y}=1}^{|\pazocal{Y}|}
  \exp\{ f(x;\theta)^\top e_{\tilde{y}} \}\Bigr),
\]
thus matching \eqref{eq:General_Exp_Family} with natural statistics \(t(y)=e_y\) and natural parameters \(f(x;\theta)\). The corresponding loss, 
\[
  \ell\bigl(y, f(x;\theta)\bigr)
  = -\log P\bigl(y \mid f(x;\theta)\bigr),
\]
is the well-known cross-entropy.

\item \textbf{Mean Squared Error (MSE).}
In a regression setting with a continuous target \(y\in\mathbb{R}\), a unit-variance Gaussian model with mean \(\mu = f(x;\theta)\) leads to
\begin{align}
    \label{eq:loss_Gaussian}
    \log P\bigl(y \mid f(x;\theta)\bigr)
    =
    -\frac{1}{2} \bigl(y - f(x;\theta)\bigr)^2
    =
    f(x;\theta)\,y
    -
    \frac{y^2}{2}
    -
    \frac{\bigl(f(x;\theta)\bigr)^2}{2}.
\end{align}
Comparing with \eqref{eq:General_Exp_Family}, this corresponds to an exponential family whose natural statistics are \(\bigl(y, y^2\bigr)\) and whose natural parameters are \(\bigl(f(x;\theta), -\frac{1}{2}\bigr)\). The negative log-likelihood here, 
\[
  \ell\bigl(y, f(x;\theta)\bigr)
  = -\log P\bigl(y \mid f(x;\theta)\bigr)
  = \frac{1}{2} \bigl(y - f(x;\theta)\bigr)^2
  \;+\; \text{(constant)},
\]
is precisely the mean squared error (MSE) loss up to an additive constant.
\end{enumerate}

\subsection{Bregman Losses}
Following \cite[Thm.~4]{banerjee2005clustering}, whenever the representation $P(y|f(x;\theta))$ correspond to a regular exponential family, then the loss $-\log(P(y|f(x;\theta)))$ can be expressed as  
\begin{align}
    -\log(P(y|f_{\theta}(x))) = d_{\varphi}(t(y), \mu(f_{\theta}(x))) + \log(b_{\varphi}(t(y))) + C
\end{align}
where $\mu(f_{\theta}(x)) = \E{t(y)}$ is the expected value of $t(y)$ using the underlying exponential family distribution, $d_{\varphi}(\cdot,\cdot)$ is a Bregman divergence and $C$ is a constant. As shown by \cite[Table~1]{banerjee2005clustering} (see also \cite{das2024unified}), this result implies that many classical losses in machine learning, including cross-entropy and mean squared error, can be viewed as special cases of Bregman divergences, and further belong to the exponential family characterization discussed in our work.  

\subsection{Properties of the Cross-Entropy and MSE Losses}
\label{app:ExpFamily_LossHessian}

We now demonstrate how the assumptions on loss minimization, Hessian boundedness, and simplified second-order gradients follow for the two loss functions introduced above.

\begin{enumerate}
\item \textbf{Cross-Entropy Loss.} 
Recall the parameterization
\[
    \log P\bigl(y \mid f(x;\theta)\bigr)
    =
    (f(x;\theta))_y
    -
    \log\Bigl(\sum_{\tilde{y}\in\pazocal{Y}} \exp\{(f(x;\theta))_{\tilde{y}} \}\Bigr),
\]
and let 
\[
  \ell\bigl(y, f(x;\theta)\bigr)
  =
  -\log P\bigl(y \mid f(x;\theta)\bigr)
  =
  \log\Bigl(\sum_{\tilde{y}\in\pazocal{Y}} \exp\{(f(x;\theta))_{\tilde{y}} \}\Bigr)
  - (f(x;\theta))_y.
\]
We focus first on the gradient of the \emph{log-probability} itself (sometimes termed the “score function”):
\begin{align*}
    \nabla_{f} \log P\bigl(y \mid f(x;\theta)\bigr) 
    &= \nabla_{f}\Bigl[(f(x;\theta))_y 
                     -
                     \log\bigl(\sum_{\tilde{y}\in\pY}
                        \exp\{(f(x;\theta))_{\tilde{y}} \}\bigr)\Bigr] \\
    &= e_y 
    - 
    \text{softmax}\bigl(f(x;\theta)\bigr),
\end{align*}
where \(e_y\) is the one-hot vector selecting entry \(y\), and 
\(\text{softmax}\bigl(f(x;\theta)\bigr)\) is the vector of class probabilities assigned by the model. 

\paragraph{Zero Gradients Under Perfect Prediction.}
For any training example \((x_i, y_i)\), if the model classifies it with perfect confidence, i.e.
\[
   \bigl(\text{softmax}(f(x_i;\hat{\theta}(\boldsymbol{1})))\bigr)_{y_i}
   = 1,
\]
then \(\nabla_{f} \log P\bigl(y_i \mid f(x_i;\hat{\theta}(\boldsymbol{1}))\bigr) = 0\).  Consequently, if the model perfectly predicts \emph{all} training labels, then all these gradients vanish simultaneously.

\paragraph{Bounded Hessian.}
Next, we show that the second derivative (the Hessian) of \(\log P\bigl(y \mid f(x;\theta)\bigr)\) with respect to \(f\) is bounded in norm.  From the above,
\[
   \nabla_{f}\log P\bigl(y \mid f(x;\theta)\bigr) 
   =
   e_y - \text{softmax}\bigl(f(x;\theta)\bigr),
\]
so taking another derivative,
\[
   \nabla^2_{f}\log P\bigl(y \mid f(x;\theta)\bigr)
   =
   -\,\nabla_{f}\bigl[\text{softmax}(f(x;\theta))\bigr].
\]
Denote 
\(\brC_f \,\triangleq\nabla_{f}\bigl[\text{softmax}(f(x;\theta))\bigr]\).  
By the well-known derivative of \(\text{softmax}\), the \((i,j)\)th entry of \(\brC_f\) is
\[
  (\brC_f)_{ij}
  =
  \frac{\partial}{\partial \bigl(f(x;\theta)\bigr)_j}
  \Bigl[\text{softmax}(f(x;\theta))_i\Bigr]
  =
  \text{softmax}(f(x;\theta))_i
  \,\bigl[\delta_{ij} - \text{softmax}(f(x;\theta))_j \bigr],
\]
which implies:
\begin{align}
    (\brC_f)_{ii} &= \text{softmax}(f(x;\theta))_i \bigl[1 - \text{softmax}(f(x;\theta))_i\bigr], \\ 
    (\brC_f)_{ij} &= -\text{softmax}(f(x;\theta))_i \text{softmax}(f(x;\theta))_j \quad (i \neq j).
\end{align}

Because each \(\text{softmax}(f(x;\theta))_i \in [0,1]\), the entries of \(\brC_f\) lie in \([-1,1]\), and indeed one can show 
\(\|\brC_f\|\) is bounded by a constant (depending only on \(\abs{\pazocal{Y}}\), not on the dimension of the parameters). Hence
\(\nabla^2_{f}\log P\bigl(y \mid f(x;\theta)\bigr) = -\brC_f\)
is also bounded in norm, establishing the desired Hessian bound.

\bigskip

\item \textbf{Mean Squared Error (MSE).}
For the MSE loss arising from a unit-variance Gaussian,
\[
    \log P\bigl(y \mid f(x;\theta)\bigr)
    =
    -\frac{1}{2}\,\bigl[y - f(x;\theta)\bigr]^2,
\]
the gradient with respect to \(f(x;\theta)\) is simply
\[
  \nabla_{f}\,\log P\bigl(y \mid f(x;\theta)\bigr)
  =
  y - f(x;\theta).
\]
Hence, if at \(\hat{\theta}(\boldsymbol{1})\) the model predictions perfectly match all responses, this gradient becomes zero for each training pair, indicating perfect minimization of the training error.

\paragraph{Bounded Hessian.}
Since
\[
   \nabla^2_{f}\,\log P\bigl(y \mid f(x;\theta)\bigr)
   =
   -\,\nabla^2_{f}\,\Bigl[\frac{1}{2}\,(y - f(x;\theta))^2\Bigr]
   =
   -\,(-\mathbf{I}_d)
   =
   \mathbf{I}_d,
\]
the Hessian with respect to \(f\) is simply the identity (for the one-dimensional \(f\)). Its norm is therefore trivially bounded by 1, and it does not depend on the dimension \(d\) of the parameters in \(\theta\).  Moreover, the Hessian can be evaluated with no complicated operations—just the constant identity matrix at each sample.

\end{enumerate}

\section{Gradient Bound for Minimizing Losses With Exponential Family Structure}
\label{s:Loss_ExpFamily}
Given a training set $\{(x_i, y_i)\}^{n}_{i=1}$ and the loss function \eqref{eq:General_Exp_Family} we derive gradient of the empirical risk \eqref{eq:WERM_Def} which we aim to minimize. To that end, we note that 
\begin{align}
   n \nabla_{\theta}L(\mD, \theta, \boldsymbol{1}) &= \nabla_{\theta}\left(\sum^{n}_{i=1} f^{\top}(x_i; \theta)t(y_i) - \log\left(\sum_{\tilde{y}\in \pY}\exp\{f^{\top}(x_i; \theta)t(\tilde{y})\}\right) + \beta(y_i)\right)\\
    &= \sum^{n}_{i=1}\left(\nabla^{\top}_{\theta}f(x_i; \theta)t(y_i) - \frac{\sum_{\tilde{y}\in\pY}\nabla^{\top}f(x_i; \theta)t(\tilde{y})\exp\{f^{\top}(x_i; \theta)t(\tilde{y})\}}{\sum_{\tilde{y}_1\in\pY}\exp\{f^{\top}(x_i; \theta)t(\tilde{y}_1)\}}\right)\\
    &= \sum^{n}_{i=1}\nabla^{\top}_{\theta}f(x_i; \theta)\left(t(y_i) - \frac{\sum_{\tilde{y}\in\pY}t(\tilde{y})\exp\{f^{\top}(x_i; \theta)t(\tilde{y})\}}{\sum_{\tilde{y}_1\in\pY}\exp\{f^{\top}(x_i; \theta)t(\tilde{y}_1)\}}\right)\\
    &= \sum^{n}_{i=1}\nabla^{\top}_{\theta}f(x_i; \theta)(t(y_i) - \Esub{\sy\sim P_{\sy|\sx=x_i;\theta}}{t(\sy)})
\end{align}
and the norm of this gradient is upper bounded by 
\begin{align}
    n\norm{\nabla_{\theta}L(\mD, \theta, \boldsymbol{1})} \leq \sum^{n}_{i=1}\norm{\nabla_{\theta}f(x_i; \theta)}\|t(y_i) - \Esub{\sy\sim P_{\sy|\sx=x_i;\theta}}{t(\sy)}\|
\end{align}
Thus, whenever the features are Lipschitz, we have
\begin{align}
    \norm{\nabla_{\theta}L(\mD, \theta, \boldsymbol{1})} \leq \frac{C_f}{n}\sum^{n}_{i=1}\|t(y_i) - \Esub{\sy\sim P_{\sy|\sx=x_i;\theta}}{t(\sy)}\|
\end{align}
and we expect this upper bound to be small at the minimizer $\theta = \hat{\theta}(\boldsymbol{1})$.

\section{Proof of Closeness of $\hat{\theta}(\boldsymbol{1})$ and $\hat{\theta}(\boldsymbol{1}^{n\backslash i})$}
\label{app:Closeness_Weighted}
We will use the next lemma throughout our proofs. 
\begin{lem}
    \label{lem:gen_strong_convex} 
    Let $\hat{\theta}(\boldsymbol{1}^{n\backslash i})$ defined as in \eqref{eq:WERM_Def} and let $\frac{1}{n}\ell(z_i, \theta)$ be a differentiable function in $\theta$ for any $z_i\in\mD$ and $\frac{1}{n}\ell(z_i, \theta) + \lambda \pi(\theta)$ be a $\mu$-strongly convex function in $\theta$ for any $z_i\in\mD$. Then, $\forall i\in[n]$ 
    \begin{align}
        \|\hat{\theta}(\boldsymbol{1}^{n\backslash i}) - \hat{\theta}(\boldsymbol{1})\| \leq \frac{2}{\mu}\cdot \underset{i\in[n]}{\max} \ \norm{\frac{1}{n}\nabla_{\theta}\ell(z_i, \hat{\theta}(\boldsymbol{1}))}
    \end{align}
\end{lem}
\begin{proof}
    Similarly to the developments from \citep[App.~B.1]{Wilson_OptimizerComparison}, we get that

    \begin{align}
        \|\hat{\theta}(\boldsymbol{1}) - \hat{\theta}(\boldsymbol{1}^{n\backslash i})\|^{2} &\leq \frac{2}{\mu}|(\hat{\theta}(\boldsymbol{1}) - \hat{\theta}(\boldsymbol{1}^{n\backslash i}))^{\top}(\nabla_{\theta}(L(\mD, \hat{\theta}(\boldsymbol{1}), \lambda, \boldsymbol{1}^{n\backslash i}) - L(\mD, \hat{\theta}(\boldsymbol{1}), \lambda, \boldsymbol{1})) )| \\
        &\leq \frac{2}{\mu n}|(\hat{\theta}(\boldsymbol{1}) - \hat{\theta}(\boldsymbol{1}^{n\backslash i}))^{\top}\nabla_{\theta}\ell(z_i, \hat{\theta}(\boldsymbol{1}))|\\
        &\leq \frac{2}{\mu n}\|\hat{\theta}(\boldsymbol{1}) - \hat{\theta}(\boldsymbol{1}^{n\backslash i})\| \|\nabla_{\theta}\ell(z_i, \hat{\theta}(\boldsymbol{1}))\|
    \end{align}
    where all the steps are by Cauchy-Schwartz
    inequality. The proof follows by maximizing over $i$. 
\end{proof}

We note that whenever $\frac{1}{n}\ell(z_i, \theta)$ is Lipschitz, the upper bound is finite. Moreover, since we normalize by $n$, the bound will go to zero with $n$ whenever the gradient grows as $o(n)$, as is usually the case in many popular machine learning problems (see \citep[Sec.~3]{giordano2019swiss}). We further note that under a more restrictive assumption that the $\ell(z_i, \theta)$ are Lipschitz then the bound is given by $\frac{2\tilde{C}}{\mu n}$ for $\tilde{C} = \underset{i\in[n]}{\max} \ \|\nabla_{\theta}\ell(z_i, \hat{\theta}(\boldsymbol{1}))\|$ and $\tilde{C} < \infty$. 
\section{Proof of \lemref{lem:Distance}}
\label{app:Prof_dist}

The proof uses the following lemma from \cite{Wilson_OptimizerComparison}:
\begin{lem}[Optimizer Comparison, \cite{Wilson_OptimizerComparison}]
    Let
    \begin{align}
        x_{\varphi_1} &\in \arg\min_x \varphi_1(x), \ \ x_{\varphi_2} \in \arg\min_x \varphi_2(x).
    \end{align}
    If each $\varphi_i$ is $\mu$-strongly convex and $\varphi_2 - \varphi_1$ is differentiable, then
    \begin{align}
    \label{eq:OptimizerComparison_Cite}
    \frac{\mu}{2}\|x_{\varphi_1} - x_{\varphi_2}\|^{2}_2
    &\leq \abs{(x_{\varphi_1} - x_{\varphi_2})^{\top}\left(\nabla(\varphi_2 - \varphi_1)(x_{\varphi_1})\right)}.
    \end{align}
\end{lem}

\begin{proof}

For the sake of the proof, we will assume that the FIM and the Hessian are invertible matrices. Under the probabilistic interpretation of the loss elements, the overall loss function for $w^n = \boldsymbol{1}^{n\backslash i}$ is 
\begin{align}
    L(\mD, \theta, \lambda, \boldsymbol{1}^{n\backslash i}) \triangleq \frac{1}{n}\sum_{j\ne i}-\log(\CP{y_j}{f(x_j;\theta)}) + \lambda \pi(\theta)
\end{align}
and we assume that $\CP{y}{f(x;\theta)}$ belongs to an exponential family whose natural parameters are the features $f(x;\theta)$, namely, $\log(\CP{y}{f(x;\theta)}) = f^{\top}(x;\theta)t(y) - \log(\sum^{\abs{\pazocal{Y}}}_{\tilde{y}=1}\exp\left\{f^{\top}(x;\theta)t(\tilde{y})\right\}) + \beta(y)$ for some natural statistics $t(y)$. 
For this model, we have 
\begin{align}
    \nabla_{\theta}\log(\CP{y}{f(x;\theta)}) = \nabla_{\theta}f(x;\theta)\nabla_{f}\log(\CP{y}{f(x;\theta)}).
\end{align}
Thus, the approximated FIM, $\brF(\mD, \theta)$, is given by
\begin{subequations}
\begin{align}
    \brF(\mD, \theta) &= \frac{1}{n}\sum^{n}_{i=1}\Esub{\sy\sim P_{\sy|\sx=x_i; \theta}}{\nabla_{\theta}f(x_i;\theta)\nabla_{f}\log(\CP{\sy}{f(x_i;\theta)})\nabla^{\top}_{f}\log(\CP{\sy}{f(x_i;\theta)})\nabla^{\top}_{\theta}f(x_i;\theta)}\\
    \label{eq:FIM_Via_Hessian_ExpFamily}
    &= \frac{1}{n}\sum^{n}_{i=1}\nabla_{\theta}f(x_i;\theta)\Esub{\sy\sim P_{\sy|\sx=x_i; \theta}}{-\nabla^{2}_{f}\log(\CP{\sy}{f(x_i;\theta)})}\nabla^{\top}_{\theta}f(x_i;\theta)\\
     \label{eq:FIM_Via_Hessian_ExpFamily_Final}
    &= -\frac{1}{n}\sum^{n}_{i=1}\nabla_{\theta}f(x_i;\theta)\nabla^{2}_{f}\log(\CP{y_i}{f(x_i;\theta)})\nabla^{\top}_{\theta}f(x_i;\theta)
\end{align}
\end{subequations}
where \eqref{eq:FIM_Via_Hessian_ExpFamily} is by using classical properties of the exponential family, and where the last equality is since the Hessian of an exponential family with respect to the natural parameters $f$ is independent of $y$ (see \appref{app:exp_family_hessian_proof}).
Moreover, we note that the Hessian of the loss is given by
\begin{align}
    \brH(\theta, \boldsymbol{1}^{n\backslash i}) &= \nabla^{2}_{\theta}L(\mD,\theta, \boldsymbol{1}^{n\backslash i}) \\
    &= \nabla^{2}_{\theta}L(\mD,\theta, \boldsymbol{1}^{n\backslash i} - \boldsymbol{1}) + \nabla^{2}_{\theta}L(\mD,\theta, \boldsymbol{1}) \\ 
    &= \nabla^{2}_{\theta}L(\mD,\theta, \boldsymbol{1}^{n\backslash i} - \boldsymbol{1}) + \frac{1}{n}\sum^{n}_{i=1}\nabla^{2}_{\theta}f(x_i;\theta)\nabla_{f}\log(\CP{y_i}{f(x_i;\theta)}) + \brF(\mD;\theta).
\end{align} 
We start by defining the next functions 

\begin{align}
    \psi_1(\theta) &\triangleq 2L(\mD, \theta, \lambda, \boldsymbol{1}^{n\backslash i}) = 2L(\mD, \theta, \boldsymbol{1}^{n\backslash i}) + 2\lambda \pi(\theta), \\ 
    \psi_2(\theta) &\triangleq -2b^{\top}(\hat{\theta}(\boldsymbol{1}), \boldsymbol{1}^{n\backslash i})\cdot(\hat{\theta}(\boldsymbol{1}) - \theta) + (\hat{\theta}(\boldsymbol{1}) - \theta)^{\top}\nabla^{2}L(\mD, \hat{\theta}(\boldsymbol{1}), \boldsymbol{1}^{n\backslash i})(\hat{\theta}(\boldsymbol{1}) - \theta) + 2\lambda \pi(\theta), \\
    \psi_3(\theta) &\triangleq -2b^{\top}(\hat{\theta}(\boldsymbol{1}), \boldsymbol{1}^{n\backslash i})\cdot(\hat{\theta}(\boldsymbol{1}) - \theta) + (\hat{\theta}(\boldsymbol{1}) - \theta)^{\top}\cdot\brF\cdot(\hat{\theta}(\boldsymbol{1}) - \theta) + 2\lambda \pi(\theta) \\
    &= (\theta - (\hat{\theta}(\boldsymbol{1}) - \brF^{-1}\cdot b(\hat{\theta}(\boldsymbol{1}), \boldsymbol{1}^{n\backslash i}) ))^{\top}\brF (\theta - (\hat{\theta}(\boldsymbol{1}) - \brF^{-1}\cdot b(\hat{\theta}(\boldsymbol{1}), \boldsymbol{1}^{n\backslash i}) )) + 2\lambda \pi(\theta) + J
\end{align}
where $J$ is a constant (which is independent of $\theta$) and $\brF$ is an abbreviation for $\brF(\mD, \hat{\theta}(\boldsymbol{1}))$. 
We first note that the minimizer of $\psi_1$ is $\hat{\theta}(\boldsymbol{1}^{n\backslash i})$ and that the minimizer of $\psi_3$ is $\tilde{\theta}(\boldsymbol{1}^{n\backslash i})$ from \eqref{eq:Distance_Approx_eq}. 

We note that Assump.~1 and Assump.~2 guarantees that the overall loss, $L$, is $\mu$-strongly convex and that the difference $L(\mD, \hat{\theta}(\boldsymbol{1}), \lambda, \boldsymbol{1}^{n\backslash i}) - L(\mD, \hat{\theta}(\boldsymbol{1}), \lambda, \boldsymbol{1})$ is differentiable. Thus, using Lem.~2, which follows by applying the optimizer comparison lemma with $L(\mD, \theta, \lambda, \boldsymbol{1}^{n\backslash i})$ and $L(\mD, \theta, \lambda, \boldsymbol{1})$ allows us to derive the following upper bound
\begin{align}
    \label{eq:Closeness_Of_Solutions}
    \|\hat{\theta}(\boldsymbol{1}) - \hat{\theta}(\boldsymbol{1}^{n\backslash i})\| \leq \frac{2}{n\mu}\cdot \|\nabla_{\theta}\ell(z_i, \hat{\theta}(\boldsymbol{1}))\| \triangleq \frac{2g_i}{\mu}.
\end{align}

The optimizer comparison lemma \citep[Lem.~1]{Wilson_OptimizerComparison} with $\psi_1$ and $\psi_3$ and Cauchy-Schwartz inequality yields
\begin{align}
    \frac{\mu}{2}\|\hat{\theta}(\boldsymbol{1}^{n\backslash i}) - \tilde{\theta}(\boldsymbol{1}^{n\backslash i})\|^2 &\leq |(\hat{\theta}(\boldsymbol{1}^{n\backslash i}) - \tilde{\theta}(\boldsymbol{1}^{n\backslash i}))^{\top}(\nabla(\psi_3 - \psi_1)(\hat{\theta}(\boldsymbol{1}^{n\backslash i})))| \\
    &\leq \|\hat{\theta}(\boldsymbol{1}^{n\backslash i}) - \tilde{\theta}(\boldsymbol{1}^{n\backslash i})\|\|(\nabla(\psi_3 - \psi_1)(\hat{\theta}(\boldsymbol{1}^{n\backslash i})))\|\\
\end{align}
We divide both sides by $\|\hat{\theta}(\boldsymbol{1}^{n\backslash i}) - \tilde{\theta}(\boldsymbol{1}^{n\backslash i})\|$, and by using the triangle inequality we get 
\begin{align}
    \label{eq:OptimizerComparison_SecondStep}
    &\frac{\mu}{2}\|\hat{\theta}(\boldsymbol{1}^{n\backslash i}) - \tilde{\theta}(\boldsymbol{1}^{n\backslash i})\| \leq \|\nabla(\psi_3 - \psi_1)(\hat{\theta}(\boldsymbol{1}^{n\backslash i}))\| \\
    \label{eq:OptimizerComparison_Inequality_Triangle}
    &\qquad\leq \|\nabla(\psi_3 - \psi_2)(\hat{\theta}(\boldsymbol{1}^{n\backslash i})) + \nabla(\psi_2 - \psi_1)(\hat{\theta}(\boldsymbol{1}^{n\backslash i}))\|\\
    &\qquad\leq \|\nabla(\psi_3 - \psi_2)(\hat{\theta}(\boldsymbol{1}^{n\backslash i}))\| + \|\nabla(\psi_2 - \psi_1)(\hat{\theta}(\boldsymbol{1}^{n\backslash i}))\| \\
    &\qquad\leq \|\nabla^{2}L(\mD, \hat{\theta}(\boldsymbol{1}), \boldsymbol{1}^{n\backslash i}) - \brF(\mD, \hat{\theta}(\boldsymbol{1}))\|\|\hat{\theta}(\boldsymbol{1}) - \hat{\theta}(\boldsymbol{1}^{n\backslash i})\| + \|(\nabla(\psi_2 - \psi_1)(\hat{\theta}(\boldsymbol{1}^{n\backslash i})))\|\\
    &\qquad = \|\nabla^{2}L(\mD, \hat{\theta}(\boldsymbol{1}), \boldsymbol{1}^{n\backslash i} - \boldsymbol{1}) + \frac{1}{n}\sum^{n}_{i=1}\nabla^{2}_{\theta}f(x_i; \hat{\theta}(\boldsymbol{1}))\nabla_{f}\log(P(y_i|f(x_i;\hat{\theta}(\boldsymbol{1}))))\|\|\hat{\theta}(\boldsymbol{1}) - \hat{\theta}(\boldsymbol{1}^{n\backslash i})\| \\
    &\qquad\qquad\qquad + \|(\nabla(\psi_2 - \psi_1)(\hat{\theta}(\boldsymbol{1}^{n\backslash i})))\|\\
    \label{eq:OptimizerComparison_Inequality_HessianLipschitz}
    &\qquad\leq \frac{g_i}{n\mu}\cdot\|-\nabla_{\theta}f(x_i;\hat{\theta}(\boldsymbol{1}))\nabla^{2}_{f}\log(P(y_i|f(x_i;\hat{\theta}(\boldsymbol{1}))))\nabla^{\top}_{\theta}f(x_i;\hat{\theta}(\boldsymbol{1})) \\
    &\qquad\qquad\qquad + \sum^{n}_{i=1}\nabla^{2}_{\theta}f(x_i;\hat{\theta}(\boldsymbol{1}))\nabla_{f}\log(P(y_i|f(x_i;\hat{\theta}(\boldsymbol{1}))))\| + \frac{Mg_{i}^{2}}{2\mu^2}
\end{align}
where \eqref{eq:OptimizerComparison_Inequality_Triangle} is since the differences $\psi_3 - \psi_2$ and $\psi_2 - \psi_1$ are differentiable and where \eqref{eq:OptimizerComparison_Inequality_HessianLipschitz} is by using the next bound: 
\begin{subequations}    
\begin{align}
     &\|(\nabla(\psi_2 - \psi_1)(\hat{\theta}(\boldsymbol{1}^{n\backslash i})))\| \\
     &= 2\|b(\hat{\theta}(\boldsymbol{1}), \boldsymbol{1}^{n\backslash i}) +\nabla^{2}L(\mD, \hat{\theta}(\boldsymbol{1}), \boldsymbol{1}^{n\backslash i})(\hat{\theta}(\boldsymbol{1}^{n\backslash i}) - \hat{\theta}(\boldsymbol{1})) - \nabla L(\mD, \hat{\theta}(\boldsymbol{1}^{n\backslash i}), \boldsymbol{1}^{n\backslash i})\| \\
     \label{eq:Equality_FullGrad}
     &= 2\|\nabla L(\mD, \hat{\theta}(\boldsymbol{1}), \boldsymbol{1}^{n\backslash i}) +\nabla^{2}L(\mD, \hat{\theta}(\boldsymbol{1}), \boldsymbol{1}^{n\backslash i})(\hat{\theta}(\boldsymbol{1}^{n\backslash i}) - \hat{\theta}(\boldsymbol{1})) - \nabla L(\mD, \hat{\theta}(\boldsymbol{1}^{n\backslash i}), \boldsymbol{1}^{n\backslash i})\| \ \ \ \\
     \label{eq:Inequality_LipschitzGrad}
     &\leq M\cdot\norm{\hat{\theta}(\boldsymbol{1}^{n\backslash i}) - \hat{\theta}(\boldsymbol{1})}^{2}\\
     & \leq \frac{4Mg_{i}^{2}}{\mu^2}
\end{align}
\end{subequations}
where \eqref{eq:Equality_FullGrad} is by the structure and the convexity and differentiability assumptions on $L$, leading to $\nabla L(\mD, \hat{\theta}(\boldsymbol{1}), \boldsymbol{1}) = 0$, \eqref{eq:Inequality_LipschitzGrad} implied by the Hessian Lipschitzness of $L$ (see also [3, Lem.~1.2.4]) and the last inequality is by Lem.~2. 

We further use the triangle inequality to get the next upper bound 
\begin{align}
    \label{eq:DistanceBound_NGD}
    \frac{\mu}{2}\|\hat{\theta}(\boldsymbol{1}^{n\backslash i}) - \tilde{\theta}(\boldsymbol{1}^{n\backslash i})\| &\leq \frac{g_i}{n\mu}(\|\nabla_{\theta}f(x_i;\hat{\theta}(\boldsymbol{1}))\nabla^{2}_{f}\log(P(y_i|f(x_i;\hat{\theta}(\boldsymbol{1}))))\nabla^{\top}_{\theta}f(x_i;\hat{\theta}(\boldsymbol{1}))\| \\
    &\qquad\qquad + \sum^{n}_{i=1}\|\nabla^{2}_{\theta}f(x_i;\hat{\theta}(\boldsymbol{1}))\nabla_{f}\log(P(y_i|f(x_i;\hat{\theta}(\boldsymbol{1}))))\|) + \frac{Mg_{i}^{2}}{2\mu^2}
\end{align}
and by using \assref{ass:LipFeatures}, \assref{ass:LossFunc_ExpFamily} and the boundedness of the Hessian of the loss relative to the features (see \appref{app:ExpFamily_LossHessian}) we get the final bound 
\begin{align}
    \label{eq:Distance_OptimalApprox_withR}
    \|\hat{\theta}(\boldsymbol{1}^{n\backslash i}) - \tilde{\theta}(\boldsymbol{1}^{n\backslash i})\| &\leq \frac{2Qg_i}{n\mu^{2}}\|\nabla_{\theta}f(x_i;\hat{\theta}(\boldsymbol{1}))\|^2 + \frac{Mg_{i}^{2}}{\mu^3} \\
    &\qquad + \frac{2g_i}{n\mu^{2}}\sum^{n}_{i=1}\|\nabla^{2}_{\theta}f(x_i;\hat{\theta}(\boldsymbol{1}))\nabla_{f}\log(P(y_i|f(x_i;\hat{\theta}(\boldsymbol{1}))))\|\\
    &\leq \frac{2Q C^2_f g_i}{n\mu^{2}} + \frac{Mg_{i}^{2}}{\mu^3} + \frac{2g_i\tilde{C}_f}{n\mu^{2}}\sum^{n}_{i=1}\|\nabla_{f}\log(P(y_i|f(x_i;\hat{\theta}(\boldsymbol{1}))))\|
\end{align}
where $Q$ is a constant s.t. $\norm{\nabla^{2}_{f}\log(\CP{y}{f(x;\theta)})} \leq Q$.  
\end{proof}

We now emphasize how the third term disappears whenever our model interpolates the training data (namely, $\ell(z_i, \hat{\theta}(\boldsymbol{1})) = 0, \forall i\in[n]$). In that case, we have $P(y_i|f(x_i;\hat{\theta}(\boldsymbol{1}))) = 1, \ \forall i\in[n]$ \footnote{In the continuous case, this amounts to $P(y_i|f(x_i;\hat{\theta}(\boldsymbol{1})))$ converging to a delta-function, concentrated around the value $y_i$}. Thus, following the notation of \appref{app:exp_family_hessian_proof} we have that $\Esub{\sy\sim P_{\sy|\sx = x_i;\hat{\theta}(\boldsymbol{1})}}{t(\sy)} = t(y_i)$ and since $\nabla_{f}\log(P(y_i|f(x_i;\hat{\theta}(\boldsymbol{1})))) = t(y_i) - \Esub{\sy\sim P_{\sy|\sx = x_i;\hat{\theta}(\boldsymbol{1})}}{t(\sy)}$ we get that the third term is zero. 

\section{Comment on \lemref{lem:Distance} When $\pi(\theta)$ is Twice-Differentiable}
\label{app:Closeness_Diff}
Whenever $\pi(\theta)$ is twice differentiable, an equivalent argument to that of \lemref{lem:Distance} can be stated without the usage of a proximal operator. Specifically, since in this case the entire loss elements $\frac{1}{n}\ell(z_i, \theta) + \lambda \pi(\theta)$ can be approximated using a second-order Taylor expansion, and a solution that uses $\brC(\hat{\theta}(\boldsymbol{1}), \boldsymbol{1}) = \brF(\mD, \hat{\theta}(\boldsymbol{1})) + \lambda \nabla^{2}\pi(\hat{\theta}(\boldsymbol{1}))$ leads to similar arguments as those from \appref{app:Prof_dist}. For this approximation we define the solution via 
\begin{align}
     \tilde{\theta}(\boldsymbol{1}^{n\backslash i}) \triangleq \hat{\theta}(\boldsymbol{1}) - (\brF(\mD,\hat{\theta}(\boldsymbol{1})) + \lambda \nabla^{2}\pi(\hat{\theta}(\boldsymbol{1})))^{-1}b(\hat{\theta}(\boldsymbol{1}), \boldsymbol{1}^{n\backslash i})
\end{align}
and a similar analysis to that of \appref{app:Prof_dist} can be carried out and to lead to similar guarantees. An example for such arguments from a similar application can be found in \citep[Thm.~2]{Wilson_OptimizerComparison}.

\section{Fisher Information Matrix for Exponential Families}
\label{app:exp_family_hessian_proof}
Using the fact that the distribution $\CP{y}{f(x;\theta)}$ belongs to an exponential family, namely 
\begin{align}
    \log(\CP{y}{f(x;\theta)}) = f^{\top}(x;\theta)t(y) - \log\left(\sum^{\abs{\pazocal{Y}}}_{\tilde{y}=1}\exp\left\{f^{\top}(x;\theta)t(\tilde{y})\right\}\right) + \beta(y),    
\end{align}
we can directly evaluate the terms $\Esub{\sy\sim P_{\sy|\sx=x_i;\theta}}{\nabla_{f}\log(\CP{\sy}{f(x;\theta)})\nabla^{\top}_{f}\log(\CP{\sy}{f(x;\theta)})}$ and $\Esub{\sy\sim P_{\sy|\sx=x_i;\theta}}{-\nabla^{2}_{f}\log(\CP{\sy}{f(x_i;\theta)})}$ to establish the desired equality. First, we find that:  
\begin{align}
    \nabla_{f}\log(\CP{y}{f(x;\theta)}) &= \nabla_{f}\left(f^{\top}(x;\theta)t(y) - \log\left(\sum_{y\in\pY}\exp\left\{f^{\top}(x;\theta)t(y)\right\}\right)\right) \\
    &= t(y) - \Esub{\sy\sim P_{\sy|\sx=x_i;\theta}}{t(\sy)}
\end{align}
and 
\begin{align}
    &\Esub{\sy\sim P_{\sy|\sx=x_i;\theta}}{\nabla_{f}\log(\CP{\sy}{f(x;\theta)})\nabla^{\top}_{f}\log(\CP{\sy}{f(x;\theta)})} \\
    &\qquad\qquad= \Esub{\sy\sim P_{\sy|\sx=x_i;\theta}}{(t(\sy) - \Esub{\sy\sim P_{\sy|\sx=x_i;\theta}}{t(\sy)})(t(\sy) - \Esub{\sy\sim P_{\sy|\sx=x_i;\theta}}{t(\sy)})^{\top}}.
\end{align}

Next, we observe that: 
\begin{align}
    -\nabla^{2}_{f}\log(\CP{y}{f(x;\theta)}) &= \nabla_{f}\left(\frac{\sum_{\ty\in\pY}t(\ty)\exp\left\{f^{\top}(x;\theta)t(\ty)\right\}}{\sum_{\ty_1\in\pY}\exp\left\{f^{\top}(x;\theta)t(\ty_1)\right\}}\right)\\
    &= \Esub{\sy\sim P_{\sy|\sx=x_i;\theta}}{t(\sy)t^{\top}(\sy)} - (\Esub{\sy\sim P_{\sy|\sx=x_i;\theta}}{t(\sy)})(\Esub{\sy\sim P_{\sy|\sx=x_i;\theta}}{t(\sy)})^{\top}\\
    &= \Esub{\sy\sim P_{\sy|\sx=x_i;\theta}}{(t(\sy) - \Esub{\sy\sim P_{\sy|\sx=x_i;\theta}}{t(\sy)})(t(\sy) - \Esub{\sy\sim P_{\sy|\sx=x_i;\theta}}{t(\sy)})^{\top}}.
\end{align}
Moreover, we note that this final result holds for any $y$. This concludes the proof. \ $\qedsymbol$ 

\section{Proof of \thmref{thm:objective_bound}}
\label{app:InferenceObjective_Bound}
\begin{proof}
We start by writing the Taylor expansion of $T(\tilde{\theta}(\boldsymbol{1}^{n\backslash i}), \boldsymbol{1}^{n\backslash i})$ around $\hat{\theta}(\boldsymbol{1}^{n\backslash i})$ to get \footnote{the existence of the Taylor expansion of $T$ is guaranteed by \assref{ass:Lipschitz_T}}: 
\begin{align}
    \label{eq:T_Approx}
    T(\tilde{\theta}(\boldsymbol{1}^{n\backslash i}), \boldsymbol{1}^{n\backslash i}) &= T(\hat{\theta}(\boldsymbol{1}^{n\backslash i}), \boldsymbol{1}^{n\backslash i}) + \nabla^{\top}_{\theta}T(\hat{\theta}(\boldsymbol{1}^{n\backslash i}), \boldsymbol{1}^{n\backslash i})(\tilde{\theta}(\boldsymbol{1}^{n\backslash i}) - \hat{\theta}(\boldsymbol{1}^{n\backslash i})) \\
    &\qquad\qquad + \frac{1}{2}(\tilde{\theta}(\boldsymbol{1}^{n\backslash i}) - \hat{\theta}(\boldsymbol{1}^{n\backslash i}))^{\top}\nabla^{2}_{\theta}T(\theta_{\text{mid}}(\boldsymbol{1}^{n\backslash i}), \boldsymbol{1}^{n\backslash i})(\tilde{\theta}(\boldsymbol{1}^{n\backslash i}) - \hat{\theta}(\boldsymbol{1}^{n\backslash i}))
\end{align}
where $\theta_{\text{mid}}(\boldsymbol{1}^{n\backslash i}) = \hat{\theta}(\boldsymbol{1}^{n\backslash i}) + \kappa \cdot (\tilde{\theta}(\boldsymbol{1}^{n\backslash i}) - \hat{\theta}(\boldsymbol{1}^{n\backslash i}))$ for some $\kappa \in [0, 1]$.  
By \eqref{eq:T_Approx} and by the Lipschitz assumptions on $T$ we get 
\begin{subequations}    
\begin{align}
    \label{eq:Score_UpperBound_Approx}
    &\|T(\tilde{\theta}(\boldsymbol{1}^{n\backslash i}), \boldsymbol{1}^{n\backslash i}) - T(\hat{\theta}(\boldsymbol{1}^{n\backslash i}), \boldsymbol{1}^{n\backslash i})\|\\
    &\qquad = \|\nabla^{\top}_{\theta}T(\hat{\theta}(\boldsymbol{1}^{n\backslash i}), \boldsymbol{1}^{n\backslash i})(\tilde{\theta}(\boldsymbol{1}^{n\backslash i}) - \hat{\theta}(\boldsymbol{1}^{n\backslash i})) \\
    &\qquad\qquad\qquad\qquad + \frac{1}{2}(\tilde{\theta}(\boldsymbol{1}^{n\backslash i}) - \hat{\theta}(\boldsymbol{1}^{n\backslash i}))^{\top}\nabla^{2}_{\theta}T(\theta_{\text{mid}}(\boldsymbol{1}^{n\backslash i}), \boldsymbol{1}^{n\backslash i})(\tilde{\theta}(\boldsymbol{1}^{n\backslash i}) - \hat{\theta}(\boldsymbol{1}^{n\backslash i}))\|\\
    \label{eq:Score_UpperBound_Bound1}
    &\qquad \leq \|\nabla_{\theta} T(\hat{\theta}(\boldsymbol{1}^{n\backslash i}), \boldsymbol{1}^{n\backslash i})\|\|\tilde{\theta}(\boldsymbol{1}^{n\backslash i}) - \hat{\theta}(\boldsymbol{1}^{n\backslash i})\| \\
    &\qquad\qquad\qquad\qquad + \frac{1}{2}\|\nabla^{2}_{\theta} T(\theta_{\text{mid}}(\boldsymbol{1}^{n\backslash i}), \boldsymbol{1}^{n\backslash i})\|_{\text{op}}\|\tilde{\theta}(\boldsymbol{1}^{n\backslash i}) - \hat{\theta}(\boldsymbol{1}^{n\backslash i})\|^{2}\ \ \ \ \ \\
    \label{eq:Score_UpperBound_Bound_final}
    &\qquad \leq C_{T_1}\|\tilde{\theta}(\boldsymbol{1}^{n\backslash i}) - \hat{\theta}(\boldsymbol{1}^{n\backslash i})\| + \frac{1}{2}C_{T_2}\|\tilde{\theta}(\boldsymbol{1}^{n\backslash i}) - \hat{\theta}(\boldsymbol{1}^{n\backslash i})\|^{2}.
\end{align}
\end{subequations}

The proof is completed by substituting \eqref{eq:DistBound_Lem} into \eqref{eq:Score_UpperBound_Bound_final}. To prove \eqref{eq:Taylor_Proof}, we write the expansion of $T(\hat{\theta}(\boldsymbol{1}^{n\backslash i}))$ around $\hat{\theta}(\boldsymbol{1})$, to get 
\begin{align}
    &\|T(\hat{\theta}(\boldsymbol{1}^{n\backslash i}), \boldsymbol{1}^{n\backslash i}) - T(\hat{\theta}(\boldsymbol{1}), \boldsymbol{1}^{n\backslash i})  - \nabla_{\theta} T(\hat{\theta}(\boldsymbol{1}), \boldsymbol{1}^{n\backslash i})(\tilde{\theta}(\boldsymbol{1}^{n\backslash i}) - \hat{\theta}(\boldsymbol{1}))\| \\
    &\qquad = \|\nabla_{\theta} T(\hat{\theta}(\boldsymbol{1}), \boldsymbol{1}^{n\backslash i})(\tilde{\theta}(\boldsymbol{1}^{n\backslash i}) - \hat{\theta}(\boldsymbol{1}^{n\backslash i})) \\
    &\qquad\qquad\qquad\qquad + \frac{1}{2}(\hat{\theta}(\boldsymbol{1}^{n\backslash i}) - \hat{\theta}(\boldsymbol{1}))^{\top}\nabla^{2}_{\theta}T(\tilde{\theta}_{\text{mid}}, \boldsymbol{1}^{n\backslash i})(\hat{\theta}(\boldsymbol{1}^{n\backslash i}) - \hat{\theta}(\boldsymbol{1}))\|\\
    &\qquad \leq C_{T_1}\|\tilde{\theta}(\boldsymbol{1}^{n\backslash i}) - \hat{\theta}(\boldsymbol{1}^{n\backslash i})\| + \frac{1}{2}C_{T_2}\|\hat{\theta}(\boldsymbol{1}) - \hat{\theta}(\boldsymbol{1}^{n\backslash i})\|^{2}
\end{align}
where $\tilde{\theta}_{\text{mid}} = \hat{\theta}(\boldsymbol{1}^{n\backslash i}) + \kappa \cdot (\hat{\theta}(\boldsymbol{1}) - \hat{\theta}(\boldsymbol{1}^{n\backslash i}))$ for some $\kappa \in [0, 1]$. Substituting \eqref{eq:DistBound_Lem} and \eqref{eq:Closeness_Of_Solutions} concludes the proof. 
\end{proof}
\section{Proofs of \corolref{corol:LOOCV_FIM} - \corolref{corol:fairness}}
\label{app:corol_proofs}

\subsection{Proof of \corolref{corol:LOOCV_FIM}}
\label{app:CV_ACV_Fisher}
We now show how to use \thmref{thm:objective_bound} to approximate LOOCV with similar guarantees to the Hessian-based technique from \citep{Wilson_OptimizerComparison}. Throughout the proof, we will use a refined version of \eqref{eq:Score_UpperBound_Bound_final}, which requires the Lipschitzness of the $T(\cdot, \boldsymbol{1}^{n\backslash i})$ only at $\hat{\theta}(\boldsymbol{1})$. We start by defining $\text{ACV} \triangleq \frac{1}{n}\sum^{n}_{i=1}\ell(z_i, \tilde{\theta}(\boldsymbol{1}^{n\backslash i}))$ and recall that $\text{CV} \triangleq \frac{1}{n}\sum^{n}_{i=1}\ell(z_i, \hat{\theta}(\boldsymbol{1}^{n\backslash i}))$. Then, similarly to \appref{app:InferenceObjective_Bound} we get
\begin{subequations}    
\begin{align}
    &\abs{\text{ACV} - \text{CV}} \\
    &= \abs{\frac{1}{n}\sum^{n}_{i=1}\ell(z_i, \tilde{\theta}(\boldsymbol{1}^{n\backslash i})) - \ell(z_i, \hat{\theta}(\boldsymbol{1}^{n\backslash i}))}\\
    &\leq \frac{1}{n}\sum^{n}_{i=1}\abs{\ell(z_i, \tilde{\theta}(\boldsymbol{1}^{n\backslash i})) - \ell(z_i, \hat{\theta}(\boldsymbol{1}^{n\backslash i}))}\\
    \label{eq:Bound_ACV_Final}
    &\leq \frac{1}{n}\sum^{n}_{i=1}\norm{\nabla_{\theta}\ell(z_i, \hat{\theta}(\boldsymbol{1}^{n\backslash i}))}\left(\frac{2Q C^{2}_f \tilde{g}_{i}}{n^{2}\mu^{2}} + \frac{M\tilde{g}_{i}^{2}}{n^{2}\mu^3} + \frac{2\tilde{g}_i\tilde{C}_f\bar{E}_n}{n\mu^{2}}\right) \\
         &\qquad\qquad + \frac{1}{2}\text{Lip}(\nabla_{\theta} \ell(z_i, \theta))\left(\frac{2Q C^{2}_f \tilde{g}_{i}}{n^{2}\mu^{2}} + \frac{M\tilde{g}_{i}^{2}}{n^{2}\mu^3} + \frac{2\tilde{g}_i\tilde{C}_f\bar{E}_n}{n\mu^{2}}\right)^{2}\\
    \label{eq:LastStep_InferenceBound}
    &\leq \frac{1}{n}\sum^{n}_{i=1}\left(\norm{\nabla_{\theta}\ell(z_i, \hat{\theta}(\boldsymbol{1}))} + \text{Lip}(\nabla_{\theta} \ell(z_i, \theta))\left(\frac{4\tilde{g}^2_i}{n^2\mu^2}\right)\right)\left(\frac{2Q C^{2}_f \tilde{g}_{i}}{n^{2}\mu^{2}} + \frac{M\tilde{g}_{i}^{2}}{n^{2}\mu^3} + \frac{2\tilde{g}_i\tilde{C}_f\bar{E}_n}{n\mu^{2}}\right)\ \ \ \ \ \\
         &\qquad\qquad + \frac{1}{2}\text{Lip}(\nabla_{\theta} \ell(z_i, \theta))\left(\frac{2Q C^{2}_f \tilde{g}_{i}}{n^{2}\mu^{2}} + \frac{M\tilde{g}_{i}^{2}}{n^{2}\mu^3} + \frac{2\tilde{g}_i\tilde{C}_f\bar{E}_n}{n\mu^{2}}\right)^{2}
\end{align}
\end{subequations}

where \eqref{eq:Bound_ACV_Final} is by using \eqref{eq:Score_UpperBound_Bound1} together with the bound from \thmref{thm:objective_bound} and by replacing the Lipschitz constants $C_{T_1}$ and $C_{T_2}$ of the objective with the corresponding gradients from \eqref{eq:Score_UpperBound_Bound1} and \eqref{eq:LastStep_InferenceBound} is by using the Taylor expansion of $\nabla_{\theta}\ell(z_i, \hat{\theta}(\boldsymbol{1}^{n\backslash i}))$ around $\hat{\theta}(\boldsymbol{1})$ and by using \lemref{lem:gen_strong_convex}. Expanding this expression yields 
\begin{align}
    \abs{\text{ACV} - \text{CV}} &\leq \left(\frac{2QC^{2}_{f}}{\mu^{2} n^{2}} + \frac{2\tilde{C}_{f}\bar{E}_n}{\mu^{2} n}\right)\cdot \frac{1}{n}\sum^{n}_{i=1}\norm{\nabla_{\theta}\ell(z_i, \hat{\theta}(\boldsymbol{1}))}^{2} + \left(\frac{M}{\mu^{3} n^{2}}\right)\cdot \frac{1}{n}\sum^{n}_{i=1}\norm{\nabla_{\theta}\ell(z_i, \hat{\theta}(\boldsymbol{1}))}^{3} \\
    &+\left(\frac{8QC^{2}_{f}}{\mu^{4} n^{4}} + \frac{8\tilde{C}_{f}\bar{E}_n}{\mu^{4} n^{3}}\right)\cdot \frac{1}{n}\sum^{n}_{i=1}\text{Lip}(\nabla_{\theta} \ell(z_i, \theta))\norm{\nabla_{\theta}\ell(z_i, \hat{\theta}(\boldsymbol{1}))}^{3}
    \\
    &+\left(\frac{4M}{\mu^{5} n^{4}}\right)\cdot \frac{1}{n}\sum^{n}_{i=1}\text{Lip}(\nabla_{\theta} \ell(z_i, \theta))\norm{\nabla_{\theta}\ell(z_i, \hat{\theta}(\boldsymbol{1}))}^{4}
    \\
    &+ \left(\frac{2Q^{2}C^{4}_{f}}{\mu^{4}n^{4}} + \frac{2\tilde{C}^{2}_f\bar{E}^{2}_n}{\mu^{4}n^{2}}\right)\cdot \frac{1}{n}\sum^{n}_{i=1}\text{Lip}(\nabla_{\theta} \ell(z_i, \theta))\norm{\nabla_{\theta}\ell(z_i, \hat{\theta}(\boldsymbol{1}))}^{2} \\ 
    &+ \left(\frac{2QC^{2}_{f}M}{n^{4}\mu^{5}}\right)\cdot \frac{1}{n}\sum^{n}_{i=1}\text{Lip}(\nabla_{\theta} \ell(z_i, \theta))\norm{\nabla_{\theta}\ell(z_i, \hat{\theta}(\boldsymbol{1}))}^{3} \\ 
    &+ \left(\frac{M^{2}}{n^{4}\mu^{6}}\right)\cdot \frac{1}{n}\sum^{n}_{i=1}\text{Lip}(\nabla_{\theta} \ell(z_i, \theta))\norm{\nabla_{\theta}\ell(z_i, \hat{\theta}(\boldsymbol{1}))}^{4}\\ 
    &+ \left(\frac{M\tilde{C}_f\bar{E}_n}{n^{3}\mu^{5}}\right)\cdot \frac{1}{n}\sum^{n}_{i=1}\text{Lip}(\nabla_{\theta} \ell(z_i, \theta))\norm{\nabla_{\theta}\ell(z_i, \hat{\theta}(\boldsymbol{1}))}^{3}\\ 
    &+ \left(\frac{2Q\tilde{C}_fC^{2}_f\bar{E}_n}{n^{3}\mu^{4}}\right)\cdot \frac{1}{n}\sum^{n}_{i=1}\text{Lip}(\nabla_{\theta} \ell(z_i, \theta))\norm{\nabla_{\theta}\ell(z_i, \hat{\theta}(\boldsymbol{1}))}^{2}
\end{align}
whose decay rate is dictated by the first two terms and is given by $O\left(\frac{C^{2}_fB_{02}}{\mu^{2} n^{2}} + \frac{\tilde{C}_f \bar{E}_nB_{02}}{\mu^{2} n}+ \frac{MB_{03}}{\mu^3 n^2}\right)$. 
\qed 

\subsection{Proof of \corolref{corol:unlearning_FIM}}
\label{app:corol_unlearning_proof} 
    The proof follows similarly to that from \citep{WilsonSuriyakumar_UnlearningProximal} by using the 
    bound $\tilde{g_i} \leq G$ in \eqref{eq:DistBound_Lem} and then using the Gaussian mechanism for differential privacy \citep[App.~A]{Dwork_AlgFoundationd_DP}.     
    \qed

We note that \corolref{corol:unlearning_FIM} parallels a similar result to that of \propref{prop:unlearning_wilson}, with different Lipschitz constants and with an additional term that depends on $\bar{E}_n$.
\subsection{Proof of \corolref{corol:Attributed_FIM}}
\label{app:corol_Attributed_FIM_proof}
    The proof is by substituting $\tilde{g}_i = \|\nabla_{\theta}\ell(z_i, \hat{\theta}(\boldsymbol{1}))\|$ in \eqref{eq:Objective_bound_final} and \eqref{eq:Taylor_Proof} and maximizing over $i$. 
    \qed 

We note that this proof parallels a similar result to that of \propref{prop:attributions}, with two additional terms: one that depends on $\bar{E}_n$ and the other that depends on the Lipschitz coefficient of the features $C_f$. 

\subsection{Proof of \corolref{corol:fairness}}
\label{app:corol_Fairness}
By using the definition of $T$ from \eqref{eq:DP_Def} and using the linearity of expectation and the triangle inequality we get that the Lipschitz coefficient of $T$ from \eqref{eq:DP_Def}, $C_{T_1}$, is given by $2C_f$. Then, the proof follows by substituting $\tilde{g}_i = \|\nabla_{\theta}\ell(z_i, \hat{\theta}(\boldsymbol{1}))\|$ in \eqref{eq:Objective_bound_final} and maximizing over $i$. 
\qed 

\section{The Connection Between Hessian-based IF and AFIF}
\label{app:dist_bound_IJ_F}
We now present two results that establish a connection between our AFIF framework and the Hessian-based influence function. First, we will prove that the Hessian-based solution $\tilde{\theta}^{\mathrm{IJ}}(\boldsymbol{1}^{n\backslash i})$ and our FIM-based solution $\tilde{\theta}^{\mathrm{IJ,AF}}(\boldsymbol{1}^{n\backslash i})$ are close. Then, we will prove that a similar statement holds with regard to the inference objective, namely, $T(\tilde{\theta}^{\mathrm{IJ}}(\boldsymbol{1}^{n\backslash i}))$ and $T(\tilde{\theta}^{\mathrm{IJ,AF}}(\boldsymbol{1}^{n\backslash i}))$ are also closed. These findings suggest that, while our development relies on assumptions that are rarely met in practical applications, the AFIF can effectively replace the Hessian-based IF without altering the conclusions typically drawn from the latter.

\subsection{Proof of Closeness of $\tilde{\theta}^{\mathrm{IJ}}(\boldsymbol{1}^{n\backslash i})$ and $\tilde{\theta}^{\mathrm{IJ,AF}}(\boldsymbol{1}^{n\backslash i})$}
\label{app:Proof_of_distance_bound_IJ_AFIF}
We will prove a slightly modified result that correspond to the definitions from \citep{Bae_PBRF_Influence}; namely, $\pi(\theta) = \norm{\theta}^2$ and where $\tilde{\theta}^{\mathrm{IJ}}(\boldsymbol{1}^{n\backslash i})$ and $\tilde{\theta}^{\mathrm{IJ,AF}}(\boldsymbol{1}^{n\backslash i})$ are defined with the regularized matrices and where $\lambda$ is chosen s.t. $\brH(\hat{\theta}(\boldsymbol{1}), \boldsymbol{1}) + \lambda \brI\succeq 0$ and $\brF(\mD, \hat{\theta}(\boldsymbol{1})) + \lambda \brI \succeq 0$. 

\begin{proof}
The proof follows similarly to \appref{app:Prof_dist}. We define the functions 
\begin{align}
    \label{eq:Auxiliary_Func_Def_IJ_F}
    \psi_1 &= -2b^{\top}(\hat{\theta}(\boldsymbol{1}), \boldsymbol{1}^{n\backslash i})(\hat{\theta}(\boldsymbol{1}^{n\backslash i}) - \theta) + (\hat{\theta}(\boldsymbol{1}^{n\backslash i}) - \theta)^{\top}(\brH(\hat{\theta}(\boldsymbol{1}), \boldsymbol{1}) + \lambda \brI)(\hat{\theta}(\boldsymbol{1}^{n\backslash i}) - \theta), \\
    \psi_2 &= -2b^{\top}(\hat{\theta}(\boldsymbol{1}), \boldsymbol{1}^{n\backslash i})(\hat{\theta}(\boldsymbol{1}^{n\backslash i}) - \theta) + (\hat{\theta}(\boldsymbol{1}^{n\backslash i}) - \theta)^{\top}(\brF(\mD, \hat{\theta}(\boldsymbol{1})) + \lambda \brI)(\hat{\theta}(\boldsymbol{1}^{n\backslash i}) - \theta)
\end{align}
whose minimizers correspond to $\tilde{\theta}^{\text{IJ}}(\boldsymbol{1}^{n\backslash i})$ and $\tilde{\theta}^{\mathrm{IJ,AF}}(\boldsymbol{1}^{n\backslash i})$ with regularized matrices, respectively. By our PSD assumption, we note that $\psi_1$ and $\psi_2$ are strongly convex, and we denote the strong convexity constant by $\mu$. We further assume that \assref{ass:LossFunc_ExpFamily} and \assref{ass:LipFeatures} hold. Then, using \citep[Lem.~1]{Wilson_OptimizerComparison} we get that 
\begin{align}
    &\frac{\mu}{2}\|\tilde{\theta}^{\text{IJ}}(\boldsymbol{1}^{n\backslash i}) - \tilde{\theta}^{\mathrm{IJ,AF}}(\boldsymbol{1}^{n\backslash i})\|^{2} \\
    &\qquad\leq \|\nabla(\psi_2 - \psi_1)(\tilde{\theta}^{\mathrm{IJ,AF}}(\boldsymbol{1}^{n\backslash i}))\| \\
    &\qquad\leq \|\brH(\hat{\theta}(\boldsymbol{1}), \boldsymbol{1}) - \brF(\mD, \hat{\theta}(\boldsymbol{1}))\|\|\hat{\theta}(\boldsymbol{1}^{n\backslash i}) - \tilde{\theta}^{\mathrm{IJ,AF}}(\boldsymbol{1}^{n\backslash i})\|\\
    &\qquad\leq \|\frac{1}{n}\sum^{n}_{i=1}\nabla^{2}_{\theta}f(x_i; \hat{\theta}(\boldsymbol{1}))\nabla_{f}\log(P(y_i|f(x_i;\hat{\theta}(\boldsymbol{1}))))\|\|\hat{\theta}(\boldsymbol{1}^{n\backslash i}) - \tilde{\theta}^{\mathrm{IJ,AF}}(\boldsymbol{1}^{n\backslash i})\| \\
    &\qquad\leq \frac{\tilde{C}_f}{n}\sum^{n}_{i=1}\|\nabla_{f}\log(P(y_i|f(x_i;\hat{\theta}(\boldsymbol{1}))))\|\|\hat{\theta}(\boldsymbol{1}^{n\backslash i}) - \tilde{\theta}^{\mathrm{IJ,AF}}(\boldsymbol{1}^{n\backslash i})\|\\
    \label{eq:IJ_AFIJ_LastBound}
    &\qquad= \tilde{C}_f\bar{E}_n\|\hat{\theta}(\boldsymbol{1}^{n\backslash i}) - \tilde{\theta}^{\mathrm{IJ,AF}}(\boldsymbol{1}^{n\backslash i})\|.
\end{align}
The final distance bound can be achieved by substituting \eqref{eq:DistBound_Lem} into \eqref{eq:IJ_AFIJ_LastBound}. 
\end{proof}

We note that this bound tells us that the Hessian-based solution and the FIM-based solution are close up to a term that depends on $\bar{E}_n$ times the distance of the error $\|\hat{\theta}(\boldsymbol{1}^{n\backslash i}) - \tilde{\theta}^{\mathrm{IJ,AF}}(\boldsymbol{1}^{n\backslash i})\|$, and gives further insight upon the empirical usage of the FIM in influence assessment tasks as done in \citep{Bae_PBRF_Influence}. 

\subsection{Proof of Closeness of $T(\tilde{\theta}^{\mathrm{IJ}}(\boldsymbol{1}^{n\backslash i}), \boldsymbol{1}^{n\backslash i})$ and $T(\tilde{\theta}^{\mathrm{IJ, AF}}(\boldsymbol{1}^{n\backslash i}), \boldsymbol{1}^{n\backslash i})$}
We will now proceed to prove a stronger result, claiming that the distance between the inference objective evaluated on $T(\tilde{\theta}^{\mathrm{IJ}}(\boldsymbol{1}^{n\backslash i}), \boldsymbol{1}^{n\backslash i})$ and on $T(\tilde{\theta}^{\mathrm{IJ, AF}}(\boldsymbol{1}^{n\backslash i}), \boldsymbol{1}^{n\backslash i})$ is small, further justifying the utility of the FIM-based influence measurement. 

\begin{proof}
    The proof follows similarly to the proof of \thmref{thm:objective_bound}. Assume that \assref{ass:LossFunc_ExpFamily}, \assref{ass:Lipschitz_T} and \assref{ass:LipFeatures} hold. Then, similarly to the proof from \appref{app:InferenceObjective_Bound} we use the Taylor expansion of $T(\tilde{\theta}^{\mathrm{IJ}}(\boldsymbol{1}^{n\backslash i}), \boldsymbol{1}^{n\backslash i})$ around $\tilde{\theta}^{\mathrm{IJ, AF}}(\boldsymbol{1}^{n\backslash i})$ to get
    \begin{align}
        &\|T(\tilde{\theta}^{\mathrm{IJ}}(\boldsymbol{1}^{n\backslash i}),  \boldsymbol{1}^{n\backslash i}) - T(\tilde{\theta}^{\mathrm{IJ, AF}}(\boldsymbol{1}^{n\backslash i}), \boldsymbol{1}^{n\backslash i})\|\\
        &\qquad = \|\nabla^{\top}_{\theta}T(\tilde{\theta}^{\mathrm{IJ, AF}}(\boldsymbol{1}^{n\backslash i}), \boldsymbol{1}^{n\backslash i})(\tilde{\theta}^{\mathrm{IJ, AF}}(\boldsymbol{1}^{n\backslash i}) - \tilde{\theta}^{\mathrm{IJ}}(\boldsymbol{1}^{n\backslash i})) \\
        &\qquad\qquad\qquad + \frac{1}{2}(\tilde{\theta}^{\mathrm{IJ, AF}}(\boldsymbol{1}^{n\backslash i}) - \tilde{\theta}^{\mathrm{IJ}}(\boldsymbol{1}^{n\backslash i}))^{\top}\nabla^{2}_{\theta}T(\theta_{\text{mid}}, \boldsymbol{1}^{n\backslash i})(\tilde{\theta}^{\mathrm{IJ, AF}}(\boldsymbol{1}^{n\backslash i}) - \tilde{\theta}^{\mathrm{IJ}}(\boldsymbol{1}^{n\backslash i}))\|\\
        &\qquad \leq \|\nabla_{\theta} T(\tilde{\theta}^{\mathrm{IJ, AF}}(\boldsymbol{1}^{n\backslash i}), \boldsymbol{1}^{n\backslash i})\|\|\tilde{\theta}^{\mathrm{IJ, AF}}(\boldsymbol{1}^{n\backslash i}) - \tilde{\theta}^{\mathrm{IJ}}(\boldsymbol{1}^{n\backslash i})\| \\
        &\qquad\qquad + \frac{1}{2}\|\nabla^{2}_{\theta} T(\theta_{\text{mid}}, \boldsymbol{1}^{n\backslash i})\|_{\text{op}}\|\tilde{\theta}^{\mathrm{IJ, AF}}(\boldsymbol{1}^{n\backslash i}) - \tilde{\theta}^{\mathrm{IJ}}(\boldsymbol{1}^{n\backslash i})\|^{2}\\
        &\qquad \leq C_{T_1}\|\tilde{\theta}^{\mathrm{IJ, AF}}(\boldsymbol{1}^{n\backslash i}) - \tilde{\theta}^{\mathrm{IJ}}(\boldsymbol{1}^{n\backslash i})\| + \frac{1}{2}C_{T_2}\|\tilde{\theta}^{\mathrm{IJ, AF}}(\boldsymbol{1}^{n\backslash i}) - \tilde{\theta}^{\mathrm{IJ}}(\boldsymbol{1}^{n\backslash i})\|^{2}\\
        &\qquad \leq \bar{C}_f\bar{E}_nC_{T_1}\|\hat{\theta}(\boldsymbol{1}^{n\backslash i}) - \tilde{\theta}^{\mathrm{IJ,AF}}(\boldsymbol{1}^{n\backslash i})\| + \frac{1}{2}C_{T_2}(\bar{C}_f\bar{E}_n)^{2}\|\hat{\theta}(\boldsymbol{1}^{n\backslash i}) - \tilde{\theta}^{\mathrm{IJ,AF}}(\boldsymbol{1}^{n\backslash i})\|^{2}
    \end{align}
\end{proof}
We note that this bound goes to zero whenever the quality of our approximation $\tilde{\theta}^{\mathrm{IJ,AF}}(\boldsymbol{1}^{n\backslash i})$ improves and whenever $\bar{E}_n \to 0$.

\section{Experimental Details}
\label{app:ExpDetails}

All experiments were implemented using the PyTorch \citep{Pytorch_paper} framework. The experiments from \secref{s:Exp_Fairness} and \secref{s:Exp_Influence} ran on NVIDIA A100 GPU, and the experiments from \secref{s:Exp_CV} ran on NVIDIA Tesla T4 GPU, demonstrating a consistent improvement in computational time across different GPU platforms. The datasets and models used in our experiments are detailed below.

\subsection{Datasets}
\label{s:Datasets}
\subsubsection{Adult}
We utilized the Adult dataset \citep{adult_2} from 

\url{https://archive.ics.uci.edu/ml/machine-learning-databases/adult}, to perform the task of predicting whether an individual’s income is more than 50,000\$ using $14$ demographic features such as age, education, marital status, and country of origin. We aimed to keep sex as a sensitive attribute that requires fairness. The dataset contains 48,842 samples, divided into 32561 train samples and 16281 test samples. During pre-processing, we remove the sensitive attribute from the set of input features, discard rows with any missing data, convert textual values to categorical ones, and normalize the numerical data using the \texttt{StandardScaler()} function from the \texttt{sklearn.preprocessing} module. These pre-processing steps are consistent with those used in previous analyses of this dataset (e.g., \citep{AbhinShah_Fairness}). We randomized the sample order by enabling the \texttt{Shuffle} option when creating the Dataloaders using \texttt{torch.utils.data.DataLoader}, ensuring the data was shuffled before being split into batches. The DP was estimated on the training data. 

\subsubsection{Insurance}
We utilized the insurance dataset \citep{lantz2019machine} from  \url{https://www.kaggle.com/datasets/teertha/ushealthinsurancedataset}, for predicting the total annual medical expenses of individuals using 5 demographic features from the U.S. Census Bureau. The sensitive attribute is sex. During pre-processing, we remove the sensitive attribute from the set of input features and further standardized the features to be between the range $0$ to $1$. The data has 1,338 samples with 676 males and 662 females. We use a train-test split ratio 0.8:0.2. We randomized the sample order by enabling the \texttt{Shuffle} option when creating the Dataloaders using \texttt{torch.utils.data.DataLoader}, ensuring the data was shuffled before being split into batches. The $\chi^2$ was estimated on the training data. 

\subsubsection{Crime}
We utilized the crime dataset \citep{redmond2002data} from \url{https://archive.ics.uci.edu/dataset/183/communities+and+crime}, considers predicting the number of violent crimes per 100K population using socio-economic information of communities in the U.S. The sensitive attribute is the percentage of people belonging to a particular race in the community. During pre-processing, we drop all the samples with the value of sensitive attribute less than 5\% to remove any outliers. We also remove the non-predictive attributes and the sensitive attribute from the set of input features, and normalize all attributes to the standardized range of [0, 1]. The resulting data has 1,112 samples, and we use a train-test split ratio 0.8:0.2. We randomized the sample order by enabling the \texttt{Shuffle} option when creating the Dataloaders using \texttt{torch.utils.data.DataLoader}, ensuring the data was shuffled before being split into batches. The $\chi^2$ was estimated on the training data. 
s
\subsubsection{CIFAR10}
We utilized the CIFAR10 \citep{CIFAR10_Krizhevsky} dataset as provided by the \texttt{torchvision} package in PyTorch. We trained the models without data augmentation. We pre-processed the data using the next three steps: first, we resized the image to have a size of $224\times 224$ pixels. Then, we converted the images to tensors using the \texttt{transforms.ToTensor()} method. Next, the images were normalized using the \texttt{transforms.Normalize()} method. The normalization process adjusts the image data so that the pixel values have a mean of $0.4914, 0.4822$, and $0.4465$ and a standard deviation of $0.2023, 0.1994$, and $0.2010$ for the red, green, and blue channels, respectively. Lastly, we filtered the dataset by leaving only the images whose label was ``plane" or ``car". We randomized the sample order by enabling the \texttt{Shuffle} option when creating the Dataloaders using \texttt{torch.utils.data.DataLoader}, ensuring the data was shuffled before being split into batches.

\subsection{Models}
\label{app:Details_NN_Arch}
All models were trained either using a cross-entropy loss or either using an MSE loss, implemented via \texttt{torch.nn.CrossEntropyLoss()} and \texttt{torch.nn.MSELoss()}. Moreover, unless specified otherwise, all the experiments were conducted using $L_2$ regularization, namely, $\pi(\theta) = \norm{\theta}^2$, that was incorporated into the loss by the usage of weight-decay and the AdamW optimizer \citep{loshchilov2017fixing}. 

\textbf{Two-Layer Classifier:}
For the tasks described in \secref{s:Exp_Fairness} and \secref{s:Exp_CV}, we trained a two-layer fully-connected network. For the Adult dataset, we have used a softmax activation, where for the insurance and crime datasets (where the task is regression) we didn't use any activation. The activation function for the hidden layer was chosen as \texttt{SeLU} activation. We used two variants of this model:
\begin{enumerate}
    \item For the task from \secref{s:Exp_Fairness}, the width of the hidden layer was chosen to be $1000$. We trained the model for $100$ epochs using the AdamW optimizer, with a learning rate of $10^{-4}$, batches of size $100$, momentum parameters $(\beta_1, \beta_2) = (0.9, 0.999)$ and a weight-decay of $10^{-6}$. 
    
    \item For the task from \secref{s:Exp_CV}, the width of the hidden layer was chosen to be $30000$. We trained the model using the AdamW optimizer, with a learning rate of $10^{-4}$, batch size of $100$, momentum parameters $(\beta_1, \beta_2) = (0.9, 0.999)$ and a weight-decay of $10^{-8}$. We varied the number of epochs from $1$ to $10$. 
\end{enumerate}

\textbf{CNN:}
The network comprises two convolutional layers and three fully connected layers, with max pooling operations interleaved between the convolutions. The first convolution layer processes a three-channel input to produce six channels using a 5×5 kernel. This is followed by a max pooling layer with a 2×2 window. The second convolution layer takes the six-channel output and produces sixteen channels using a 5×5 kernel, and is again followed by a 2×2 max pooling layer. After the pooling operations, the output is flattened into a one-dimensional vector. This vector is then passed sequentially through three fully connected layers: the first maps the flattened vector (of size 16×53×53) to 120 units, the second reduces it from 120 to 84 units, and the final layer maps from 84 units to 2 output units. ReLU activation functions are applied after the convolution layers and the first two fully connected layers. The network was trained for $100$ epochs on a subset of the CIFAR10 dataset that contains only samples with the label ``plane" or ``car" using the AdamW optimizer with a learning rate of $10^{-5}$, a weight decay of $10^{-6}$, and the default momentum parameters $(\beta_1, \beta_2) = (0.9, 0.999)$. The model was trained with a batch size of 128. 

\textbf{ResNet18:}
We used PyTorch's pre-trained version of ResNet18, initially trained on the ImageNet dataset, as delivered in the \texttt{torchvision.models} library. A fully connected layer of size $1000 \times 2$ was added, and the whole network (the pre-trained part and the additional output layer) was fine-tuned for $10$ epochs on a subset of the CIFAR10 dataset that contains only samples with the label ``plane" or ``car" using the AdamW optimizer with a learning rate of $10^{-5}$, a weight decay of $10^{-6}$, and the default momentum parameters $(\beta_1, \beta_2) = (0.9, 0.999)$. The model was trained with a batch size of 128. 


\subsection{Details for Fairness and Unlearning}
\label{app:Fairness_Details}
We trained the model on the training set of each dataset. Using the trained model, we estimated the fairness metric (either \eqref{eq:DP_Def} or \eqref{eq:chi2_def}) on the training data and measured the influence of each sample on this metric using the Taylor series-based approximation. We then selected all indices with a positive influence and unlearned them from the model by applying \eqref{eq:Distance_Approx_eq}. The vector $w^{n}$ used for generating $\tilde{\theta}^{\mathrm{IJ, AF}}$ had zeros at the selected indices and ones elsewhere. In all cases, we have measured influence using the \texttt{LiSSA} algorithm (see detailed description in \cite{Bae_PBRF_Influence, HazanAgrawal_Lissa}), where we have set the depth parameter to $2000$ and the number of repetitions to $3$. The parameter $\sigma$ was set to $\frac{1}{500}$ for the Adult dataset, and to $\frac{1}{2000}$ for the Insurance and the Crime datasets. Those parameters were tuned manually for achieving good results for both the Fisher-based and the Hessian-based techniques for each dataset. 

\subsection{Details for Cross-Validation}
\label{app:CV_Details}
We have performed a leave-$k$-out CV to estimate the test loss. To that end, we first pick a random subset of $6000$ training points and generate the leave-$k$-out estimator by using \eqref{eq:Distance_Approx_eq} where $\tilde{\theta}^{\mathrm{IJ, AF}}$ was generated with a vector $w^{n}$ that contain zeros for the chosen indices and ones everywhere else. Then, we estimated the loss this model has on the samples chosen by using the plug-in estimator and then averaging the loss estimates over the different samples. The final estimate was generated by repeating this process five different times and reporting the averaged estimate across the different experiments. 

\subsection{Details for Data Attribution}
\label{app:Attribution_details}
For the data attribution experiments, we used the trained model and calculated the influence scores by using the Taylor series-based approximation for the inference objective $\ell(z_{\text{test}}, \theta) - \ell(z_{\text{test}}, \hat{\theta}(\boldsymbol{1}))$, namely 
\begin{align}
    \text{IF}(z_{\text{test}}, z_i) = -\nabla^{\top}_{\theta}\ell(z_{\text{test}}, \hat{\theta}(\boldsymbol{1}))(\brC(\hat{\theta}(\boldsymbol{1}), \boldsymbol{1}))^{-1}\nabla_{\theta}\ell(z_i, \hat{\theta}(\boldsymbol{1}))
\end{align}
where $\brC(\hat{\theta}(\boldsymbol{1}), \boldsymbol{1})$ is either the Hessian or the approximated FIM. 

\section{Additional Experiments}
\label{app:additional_exps}
\subsection{Additional Experiments for \secref{s:Exp_Fairness}}
\label{app:AdditionalExps_Fairness}

To further demonstrate the usefulness of our approach, we have repeated the same experiment from \secref{s:Exp_Fairness} but with different scaling factors for the \texttt{LiSSA} algorithm. Our goal is to show that in practice the AFIF provides further more stable method for measuring influence. To that end, we have increased the scaling factor of the \texttt{LiSSA} algorithm, which controls its convergence. Following \figref{fig:combined_experiment_additional_lissa}, the Hessian-based method fails to provide a reasonable solution that corrects the model's fairness and provides solutions with inconsistent fairness metrics. On the other hand, our algorithm can reduce fairness while maintaining the model performance and further consistently outperforms the Hessian-based method regarding computational time. This further demonstrates the superiority of our algorithm in terms of computational time and further shows that it requires less hyperparameter tuning. 

\begin{figure}[t]
  \centering
  \begin{subfigure}[t]{0.32\textwidth}
    \centering
    \includegraphics[width=\textwidth]{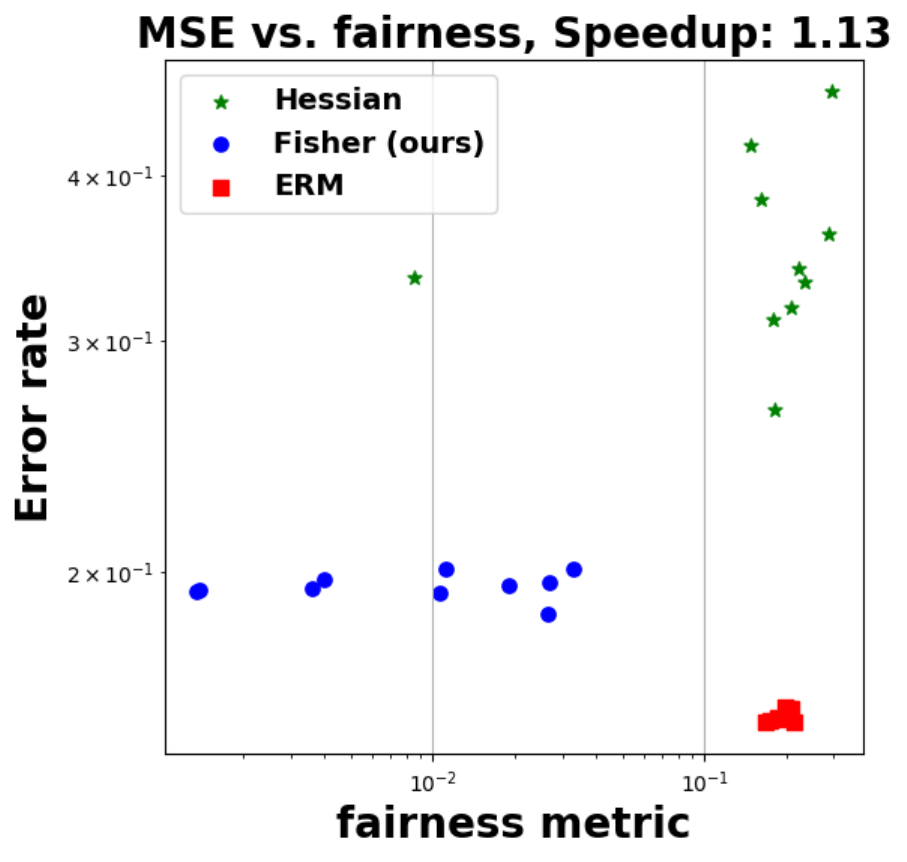}
    \subcaption{Adult: $\sigma = \frac{1}{200}$}
    \label{fig:sub_sigma_200}
  \end{subfigure}
  \hfill
  \begin{subfigure}[t]{0.32\textwidth}
    \centering
    \includegraphics[width=\textwidth]{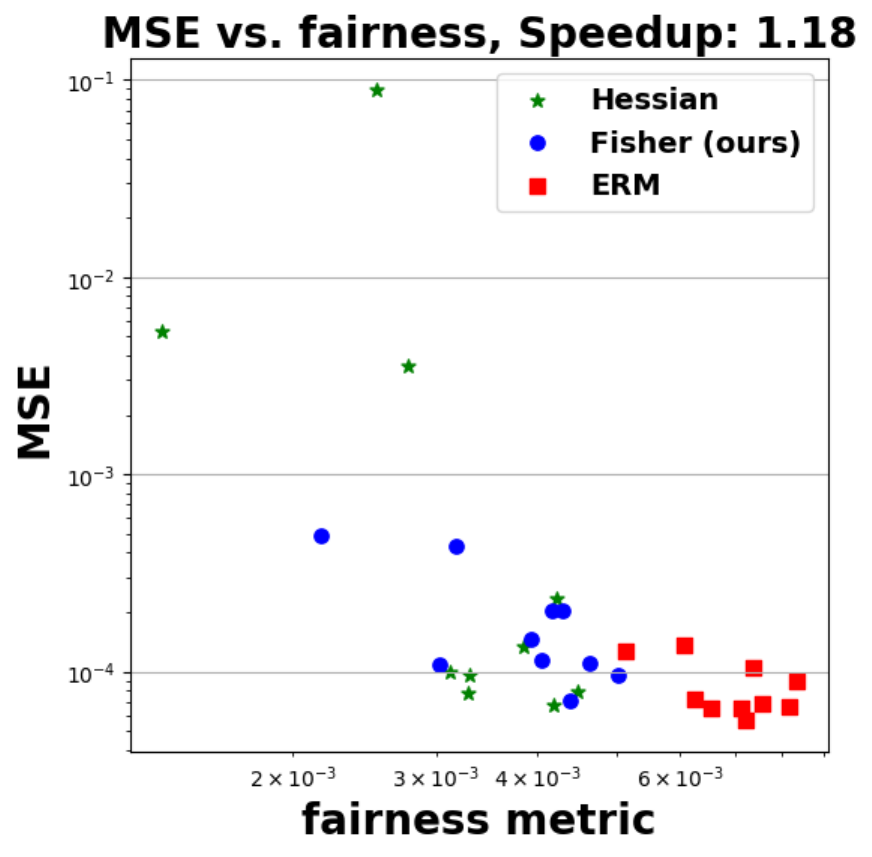}
    \subcaption{Insurance: $\sigma = \frac{1}{2000}$}
    \label{fig:sub_sigma_100}
  \end{subfigure}
  \hfill
  \begin{subfigure}[t]{0.32\textwidth}
    \centering
    \includegraphics[width=\textwidth]{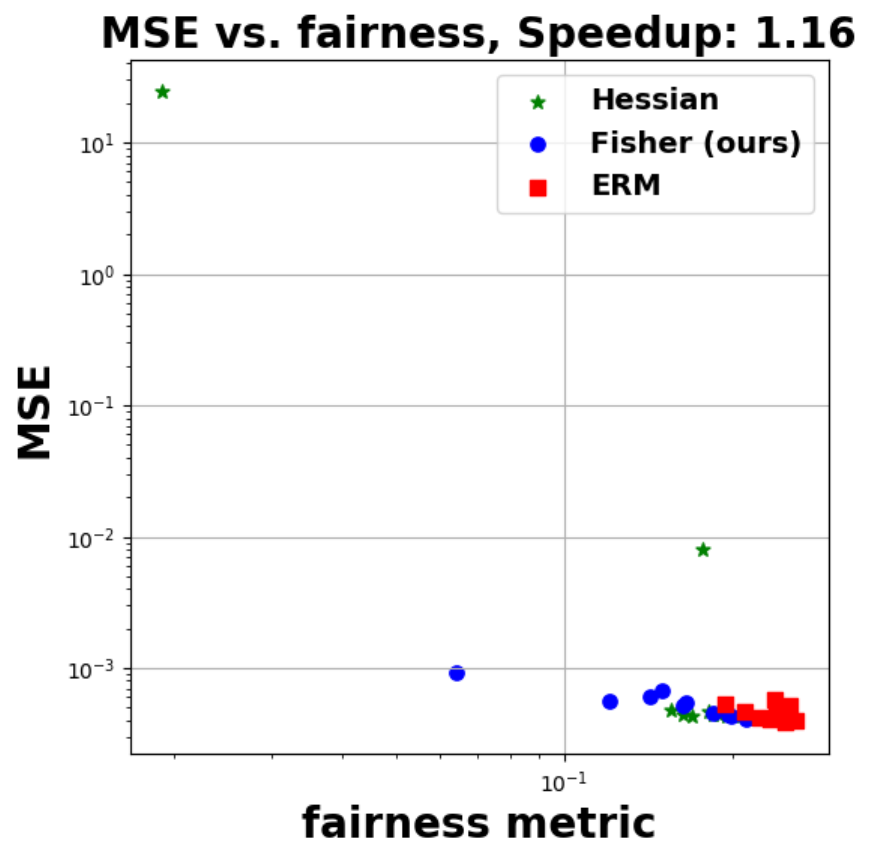}
    \subcaption{Crime: $\sigma = \frac{1}{500}$}
    \label{fig:sub_sigma_500}
  \end{subfigure}
  \caption{Additional experiment results with a reduced value of $\sigma$ of the \texttt{LiSSA} algorithm. All figures consistently shows the computational advantage of our method and further demonstrate potential instabilities of using the Hessian-based techniques, leading to cases with severely degraded performance. }
  \label{fig:combined_experiment_additional_lissa}
\end{figure}

\subsection{Additional Experiments for \secref{s:Exp_CV}}
\label{app:ExpDetails_CV}

To demonstrate our claim about the stability of the Hessian-based computations, we have repeated the same experiment from \secref{s:Exp_CV} and have decreased the parameter $\sigma$ of the \texttt{LiSSA} algorithm from $\frac{1}{500}$ to $\frac{1}{750}$. Since this parameter shrinks the inner matrix in the computations, it should help the algorithm to converge in cases where the underlying matrix is poorly conditioned. However, as is demonstrated in \figref{fig:Exp_CV_additional}, the Hessian-based CV estimator still fails to converge to a reasonable estimate of the test loss. However, we note that under this hyperparameter setting, the Fisher-based algorithm converges to a better estimate of the test loss. 
\begin{figure}[t]
    \centering
    \includegraphics[width=0.475\columnwidth]{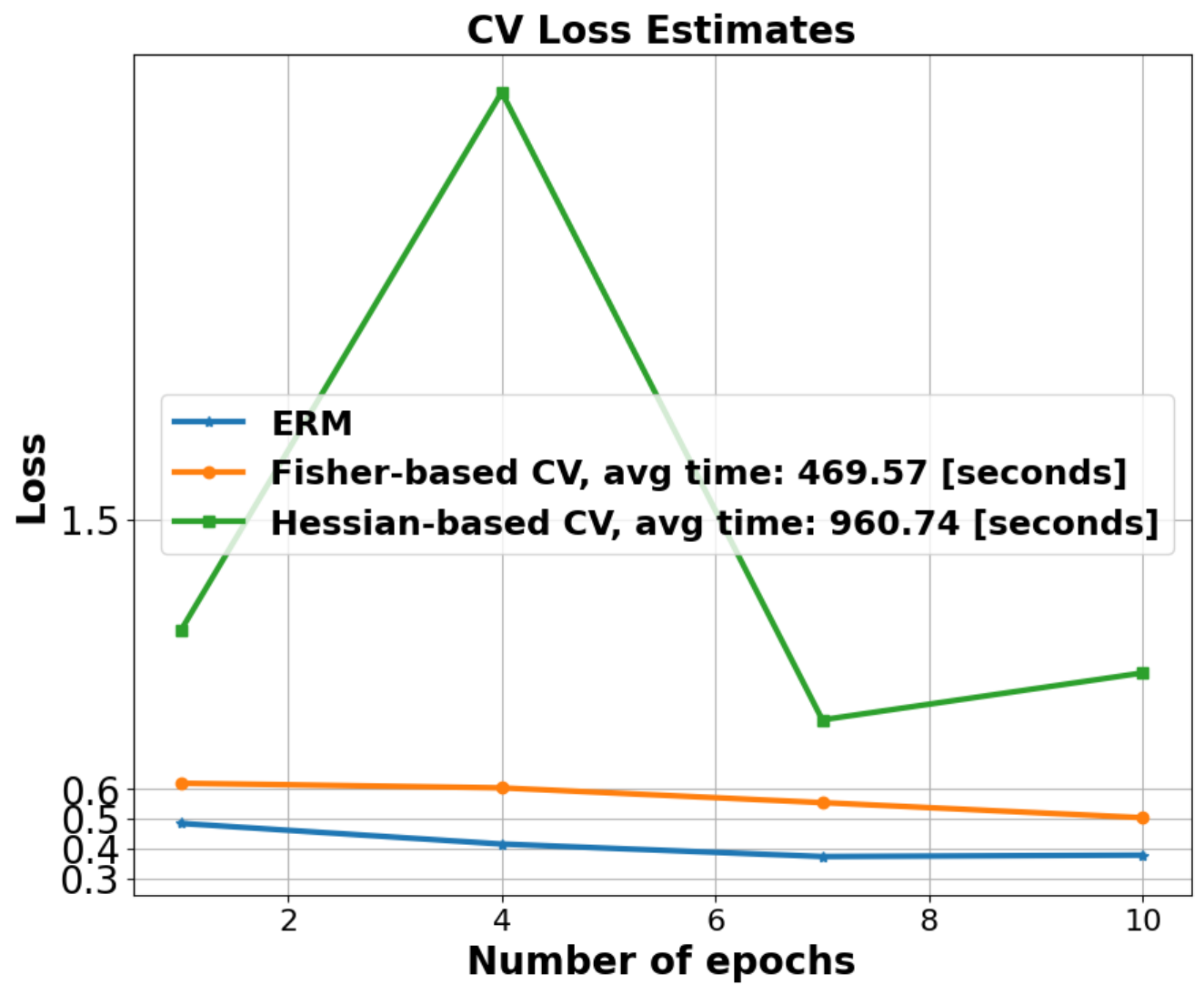}
    \caption{Additional experiment for \secref{s:Exp_CV}, where the parameter $\sigma$ of the \texttt{LiSSA} was set to $\sigma = \frac{1}{750}$.}
    \label{fig:Exp_CV_additional}
\end{figure}

\end{document}